%% file: colt2026.tex
\documentclass[12pt,cleveref]{colt2026} %

\usepackage{hyperref} 
\hypersetup{ hidelinks }%
\usepackage{amsfonts}   
\usepackage{nicefrac}    
\usepackage{stmaryrd}
\SetSymbolFont{stmry}{bold}{U}{stmry}{m}{n}
\usepackage{xspace}		%

\usepackage{mathtools}
\SetKwInput{KwInit}{Initialization}
\SetKwInput{KwDef}{Define}
\usepackage{thmtools,thm-restate}
\usepackage{bm}
\usepackage{nicefrac}
\usepackage{etoolbox}
\usepackage{tikz}
\usepackage{mathrsfs}

\AtBeginEnvironment{algorithm}{\setcounter{AlgoLine}{0}}

\newtheorem{condition}{Condition}
\crefname{condition}{condition}{conditions}

\usepackage{macros}

\makeatletter
\expandafter\let\csname enumerate*\endcsname\relax
\expandafter\let\csname endenumerate*\endcsname\relax
\makeatother
\usepackage[inline,shortlabels]{enumitem}	%
\setlist[1]{topsep=\smallskipamount,itemsep=\smallskipamount,left=\parindent}
\setlist[2]{left=0pt}
\setlist[enumerate,1]{label=\upshape(\itshape\alph*\hspace*{1pt}\upshape)}

\DeclarePairedDelimiterXPP{\dnorm}[1]{}{\lVert}{\rVert}{_{\ast}}{#1}		%

\DeclarePairedDelimiterX{\braket}[2]{\langle}{\rangle}{#1,#2}		%
\DeclarePairedDelimiterX{\product}[2]{\langle}{\rangle}{#1,#2}		%

\DeclarePairedDelimiterX{\setdef}[2]{\{}{\}}{#1:#2}		%
\DeclarePairedDelimiterXPP{\exclude}[1]{\mathopen{}\setminus}{\{}{\}}{}{#1}

\newcommand{\debug}[1]{#1}		%

\newcommand{\refines}{\preccurlyeq}
\newcommand{\refined}{\succcurlyeq}

\newcommand{\class}{\debug C}		%

\newcommand{\trait}{\debug T}
\newcommand{\traitalt}{\trait'}
\newcommand{\traitaltalt}{\trait''}
\newcommand{\traits}{\mathscr{\trait}}
\newcommand{\subtraits}{\mathcal{\debug S}}
\newcommand{\iTrait}{\debug i}
\newcommand{\jTrait}{\debug j}
\newcommand{\nTraits}{\debug n}

\newcommand{\iLvl}{\debug j}

\DeclareMathOperator{\req}{\mathtt{\debug{prim}}}

\newcommand{\requires}{\vdash}
\newcommand{\nrequires}{\nvdash}

\newcommand{\higher}{>}

\newcommand{\lattice}{\debug \Lambda}

\newcommand{\elem}{\debug a}		%
\newcommand{\nElems}{\debug K}		%
\newcommand{\elems}{\mathcal{\debug A}}		%

\newcommand{\alt}[1]{#1'}		%

\DeclareMathOperator*{\union}{\bigcup}		%

\newcommand{\txs}{\textstyle}		%

\DeclareMathOperator{\conj}{\mathtt{\debug{combine}}}
\DeclareMathOperator{\cspan}{\mathsf{\debug \Lambda}}

\newcommand{\one}{\mathsf{\debug 1}}
\newcommand{\zero}{\mathsf{\debug 0}}

\newcommand{\new}[1]{#1^{+}}

\newcommand{\brand}{\debug{\ensuremath{\mathtt{brand}}}\xspace}
\newcommand{\kolor}{\debug{\ensuremath{\mathtt{color}}}\xspace}
\newcommand{\model}{\debug{\ensuremath{\mathtt{model}}}\xspace}
\newcommand{\type}{\debug{\ensuremath{\mathtt{type}}}\xspace}
\newcommand{\yeer}{\debug{\ensuremath{\mathtt{year}}}\xspace}

\title[Leveraging Similarities in Multi-Armed Bandits]{Leveraging Similarities in Multi-Armed Bandits}
\usepackage{times}

\coltauthor{%
 \Name{Khaled Eldowa} \Email{khaled.eldowa@inria.fr}\\
 \addr Univ. Grenoble Alpes, 
 Inria, CNRS, Grenoble INP, LJK,
38000 Grenoble, France
 \AND
  \Name{Thibaud Rahier} \Email{t.rahier@criteo.com}\\
 \addr Criteo AI Lab, Paris, France
 \AND
  \Name{Augustin Cablant} \Email{augustin.cablant@inria.fr}\\
 \addr Univ. Grenoble Alpes, 
Inria, CNRS, Grenoble INP, LJK,
38000 Grenoble, France
 \AND
  \Name{Panayotis Mertikopoulos} \Email{panayotis.mertikopoulos@imag.fr}\\
 \addr Univ. Grenoble Alpes, Inria, CNRS,  Grenoble INP, LIG, 38000 Grenoble, France
 \AND
 \Name{Pierre Gaillard} \Email{pierre.gaillard@inria.fr}\\
 \addr Univ. Grenoble Alpes, 
 Inria, CNRS, Grenoble INP, LJK,
38000 Grenoble, France%
}

\begin{document}
\maketitle

\begin{abstract}%
In many online learning and bandit problems, the actions we consider possess inherent similarities\textendash for instance because they share latent traits, tags, or hierarchical structure. We study online learning with a similarity-structured action set, encoded by a rooted tree whose leaves are the actions and whose levels quantify how closely two actions are related. The loss sequence is assumed tree-compatible: losses of similar actions are constrained to be close. We establish an impossibility result showing that usual one-point bandit feedback cannot, in general, leverage range or tree-induced similarity, even under very strong similarity constraints. We then provide a unified set of algorithms which adapt to a wide range of richer feedback models, from semi-bandit feedback down to multi-point bandit protocols, including the minimal two-point feedback setting. We show these algorithms exhibit best-of-both-worlds guarantees and provably exploit action similarities by replacing the number of actions $K$ by a similarity-aware effective number of actions $K_{\mathrm{eff}}$ in the regret bounds. As an application, we show that under two-point feedback, it is possible to achieve $\sqrt{T}$ regret in Lipschitz bandits when $d \leq 2$.
\end{abstract}


\section{Introduction}

\input{1_Introduction}

\section{Preliminaries} \label{sec:preliminaries}

\input{2_Preliminaries}

\section{Lower Bound: The Insufficiency of One-Point Feedback} \label{sec:lowerbound}
\input{5_Alt_LowerBound}

\section{Basic Results for the Standard Bandit Problem}
To properly contextualize the techniques used in the coming sections, we lay out here a brief summary of some basic tools and results pertaining to the standard bandit problem.
A standard approach for tackling said problem is through a reduction to online convex optimization, employing an algorithm like follow-the-regularized-leader (FTRL) or mirror descent over the probability simplex of the actions as a decision set, see, e.g., \citep{orabona2025}.
In particular, if we fix a sequence $(z_t)_{t \geq 1}$ of vectors in $\R^K$ and a sequence of regularizers $(\psi_t)_{t \geq 1}$ with $\psi_t \colon \R^K \rightarrow \R$ being a convex function whose domain includes $\Delta_K$, the FTRL algorithm produces a sequence of distributions over actions $(p_t)_{t \geq 1}$ given by
\[
    p_1 = \argmin_{p \in \Delta_K} \psi_1(p) \quad \text{and} \quad p_{t} = \argmin_{p \in \Delta_K} \ban{\summ_{s \leq t-1} z_s,p} + \psi_{t}(p) \:\:\text{for}\, t\geq 2\,,
\]
which one can utilize by sampling $A_t$ from $p_t$.
A standard choice for $z_t(a)$ is %
$\indicator{a = A_t} y_t(a) /  p_t(a)$, which is an unbiased estimate of $y_t(a)$.
The last ingredient to specify is the regularizer $\psi_t$. 
A common form in bandit problems is $\psi_t(p) \coloneqq (1/\eta_t) \sum_{a \in \cA} f(p(a))$ for $p \in \Delta_K$, where $f$ is a one-dimensional convex function.
Assuming $f$ is Legendre (see \citealt[Chapter 26]{rockafellar1970}), it holds for any $q \in \Delta_K$ that 
\begin{multline} \label{eq:ftrl-bias-variance-bound} \textstyle
        \sum_{t=1}^T \inprod{z_t, p_t - q} \leq \psi_{T+1}(q)-\psi_{1}(p_1) + \sum_{t=1}^T \psi_{t}(p_{t+1}) - \psi_{t+1}(p_{t+1})
        \\\textstyle + \sum_{t=1}^T ({1}/{\eta_t}) \sum_{a \in \cA} D_{f^*}\brb{f'(p(a)) -  \eta_t (z_t(a) - y_t(A_t))  \big\|f'(p(a))}\,,
\end{multline}
where $D_{f^*}$ is the Bregman divergence induced by the convex conjugate of $f$.
Note that the expectation of the left-hand side is the regret of the learner that samples $A_t$ from $p_t$.
The terms in the first line of the bound are collectively referred to as the bias term, which depends primarily on the magnitude of the regularizer; while the last sum is usually referred to as the variance term, this is a measure of the stability of the sequence $(p_t)_{t \geq 1}$, which depends on the magnitude of the losses and the curvature of the regularizer.

Setting $\eta_t = \eta / \sqrt{t}$ with $\eta > 0$ and $f(x) = x \ln x$ recovers the well-known EXP3 algorithm, whose variance and bias terms are respectively bounded in expectation by $\eta K \sum_{t=1}^T 1/\sqrt{t}$ and $(1/\eta) \E \sum_{t=1}^T (1/\sqrt{t}) \sum_{a} p_t(a) \ln (1/p_t(a))$ up to constants, assuming $y_t \in [0,1]^K$.
A suitable tuning of $\eta$ yields a regret bound of order $\sqrt{KT \ln K}$.
Note that the regularizer used by EXP3 is the negative Shannon entropy function.
Alternatively, one can choose $f(x) = 2 (x - \sqrt{x})$, which yields the negative $1/2$-Tsallis entropy as the regularizer.
Using this function with $\eta_t = 1 / \sqrt{t}$, the bias and variance terms are nicely balanced, both being bounded in expectation by $\E \sum_{t=1}^T (1/\sqrt{t}) \sum_{a}(\sqrt{p_t(a)} - p_t(a))$ up to constants.
This yields an improved regret bound of order $\sqrt{K T}$, which is unimprovable in the minimax sense \citep{auer2002finite}. 

Besides a minor improvement in the worst case regime, this algorithm enjoys a major advantage against easy loss sequences.
In particular, since its regret scales with the sum of the Tsallis entropy of $p_t$ at each round, one expects improved performance if the sequence $(p_t)_{t \geq 1}$ rapidly concentrates mass on a single action.
One such scenario is when the losses are drawn at every round in an i.i.d. manner from a fixed distribution.
More generally, under \Cref{condition:stochastic},
the algorithm enjoys a regret bound of order $\sum_{a \colon \Delta(a) > 0} \ln (T) / \Delta(a)$, see \citep{zimmert2021tsallis}.
It is worth noting that one can also achieve a logarithmic regret bound in the stochastic setting by equipping EXP3 with a certain adaptive sequence of learning rates $(\eta_t)_{t \geq 1}$, even if it is an inferior bound of order $K \ln (T)^2 / \min_{a \colon \Delta(a) > 0 } \Delta(a)$, see \citep{ito-bobw-graphs}.

\section{Warm-up: Full-Path Feedback}
\label{sec:semi-bandit-feedback}

In this section, we adopt the more structured model of \cite{martin2022nested}.
The purpose here is to illustrate our results, specifically, our best-of-both-worlds bounds, in a simpler setting and provide an algorithmic template to build upon in the coming sections.
To recall, compared to ours, the setting of \cite{martin2022nested} involves two interdependent restrictions.
The first, posits that the adversary selects at every round a function $\ell_t \colon \cV \rightarrow [0,1]$ assigning a loss to every node in the tree such that $\ell_t(v) \leq R_j$ for all $v \in V_j$, where $(R_j)_{j \in [L]}$ are fixed positive weights in $[0,1]$. 
The loss $y_t(a)$ of any action $a \in \cA$ is then taken as $\sum_{\ancof{v}{a}} \ell_t(v)$.
Note that our level-wise smoothness property can be imposed on this model taking $R_j = \scale_j - \scale_{j+1}$, though the class of admissible loss functions is now smaller.
The second and more crucial restriction is that the learner gets to observe all the values $(\ell_t(v))_{\ancof{v}{a}}$ upon choosing $A_t = a$.
In practical scenarios, it can be difficult to specify the $\ell_t$ function and unrealistic to assume we can access its values along the whole path from the root to the played action.
Nevertheless, we entertain this model for the moment as a starting point, providing in the process improved results compared to \cite{martin2022nested}.

In the worst case (adversarial) regime, \cite{martin2022nested} showed that it is possible to leverage the similarity structure of the problem and the extra feedback to improve upon the minimax optimal $\sqrt{K T}$ regret rate of the standard bandit problem.
In particular, they achieve a bound of order $\sqrt{\keff(R) T \log K}$, where
$
    \keff(R) \coloneqq \brb{\sum_{j \in [L]} {R}_j \sqrt{|V_j|}}^2 
$
is seen as the effective number of actions.
When $\sum_{j \in [L]} R_j = 1$, as assumed by \cite{martin2022nested}, $\keff$ is never worse than $K$, and can be much smaller when the tree has a favorable structure, an example is provided at the end of this section.
The algorithm they use to achieve this result is a nested version of the EXP3 algorithm, which can be seen as an instance of FTRL with two main features that distinguish it from the examples provided in the previous section.
Firstly, they leverage the semi-bandit feedback by using $\hat{\ell}_t(v) \coloneqq \indicator{\ancof{v}{A_t}} \ell_t(v) / p_t[v]$ as an unbiased estimate of $\ell_t(v)$, leading to the choice of $z_t(a) = \sum_{\ancof{v}{a}} \hat{\ell}_t(v)$.\footnote{We use $\indicator{\cdot}$ to denote the indicator function.}
The second feature is the use of a nested Shannon entropy regularizer defined as
$
   \psi_t(p) = ({1}/{\eta_t}) \sum_{j \in [L]} \alpha_j \sum_{v \in V_j} p[v] \ln p[v] \,,
$
which is a weighted sum (with non-negative weights $(\alpha_j)_{j \in [L]}$) of the Shannon entropy of the aggregated version of $p$ at every level.

Ultimately though, the use of this special form of regularizer is not fully justified by the results of \cite{martin2022nested}. 
In particular, they present an upper bound of order $(\keff(R) / \alpha_L) \sum_{t=1}^T \eta_t$ on the variance term of their algorithm, which holds even if $\alpha_1 = \dots = \alpha_{L-1} = 0$, i.e., the standard version of EXP3.
Their bound on the bias term takes a complex form that does involve all the weights, but is less significant as its dependence on the number of actions is logarithmic.
To gain more insight, we show that for any regularizer of the form $\psi_t(p) = ({1}/{\eta_t}) \sum_{j \in [L]} \alpha_j \sum_{v \in V_j} f(p[v])$ with $f$ being Legendre, the bound in \eqref{eq:ftrl-bias-variance-bound} holds with the last sum replaced by (see \Cref{lem:simple-variance-generic})
\begin{align*} \textstyle
    \sum_{t=1}^T {1}/{\eta_t}  \sum_{j \in [L]} \alpha_j \sum_{v \in V_j} D_{f^*}\brb{f'(p[v]) -  \eta_t \zeta_t(v) / \alpha_j  \big\|f'(p[v])} \,,
\end{align*}
where $\zeta_t \colon \cV \rightarrow \R$ is any function satisfying $\sum_{\ancof{v}{a}} \zeta_t(v) - \sum_{\ancof{w}{a'}} \zeta_t(w) = z_t(a) - z_t(a')$ for all $a,a' \in \cA$.
When $f(x) = x \ln x$, one can show that choosing $\zeta_t(v) = \indicator{v \in \cA}z_t(a)$, i.e., pushing all the weights to the last level, allows recovering the aforementioned bound of $(\keff(R) / \alpha_L) \sum_{t=1}^T \eta_t$. 
It is unclear if this can be significantly improved upon; moreover, one can safely neglect the tradeoff with the bias term (considering the tuning of the $\alpha$ weights) due to the great disparity of magnitude between bias and variance in this case (logarithmic vs linear in the number of actions).

On the other hand, picking $f(x) = 2 (x - \sqrt{x})$ and making the more natural choice of $\zeta_t(v) = \hat{\ell}_t(v)$ (up to a constant shift) balances the bias and variance bounds, as long as $\alpha_j \propto R_j$. %
Specifically, one obtains a bound of order  
$ \sum_{t=1}^{T} {1}/{\sqrt{t}} \, \E \sum_{j \in [L]} {R}_j \sum_{v \in V_j}  \sqrt{p_t[v]} - p_t[v] \,,
$
which %
is an analogue of the entropy-dependent bound discussed in the previous section for the unstructured bandit problem.
The advantage of this result is, again analogously to the discussion in the previous section, that it can be used to obtain improved bounds in more benign scenarios where the losses are drawn at every round from a fixed distribution, which was not addressed by \cite{martin2022nested}.
Crucially, the appearance of the scale parameters in this bound allows leveraging the structure of the problem to obtain improved bounds (compared to unstructured bandits) also in the stochastic case.
Imperative to this is the nested form of the regularizer, which injects the scale parameters also into the bias term.
We present then in the following theorem a best-of-both-worlds: a slight improvement of the worst-case bound of \cite{martin2022nested} and a new structure-aware logarithmic regret bound holding in the stochastic setting. Recall that $\keff(R) \coloneqq \brb{\sum_{j \in [L]} {R}_j \sqrt{|V_j|}}^2$, and define
\[
\textstyle \keffsto(R) \coloneqq \brb{ \sum_{j \in [L]} R_j \sqrt{ \Gamma(j,\Delta) }}^2 \quad  \text{and} \quad  \displaystyle
        \Gamma(j,\Delta) \coloneqq \sum_{v \in V_j \colon \min_{\descof{a}{v}} \Delta(a) \neq 0}  \frac{1}{\min_{\descof{a}{v}} \Delta(a)} \,.
\]    
\vspace*{-10pt}
\begin{restatable}{rethm}{thmsemi}\label{thm:semi-bandit-feedback}
    FTRL with $
        \psi_t(p) = \sqrt{ 5 \max\{t, 20\} } \sum_{j \in [L]} R_j \sum_{v \in V_j} p[v] - \sqrt{p[v]}  
    $ and $z_t(a) = \sum_{\ancof{v}{a}} \indicator{\ancof{v}{A_t}} {\ell_t(v)} /{p_t[v]}$
    satisfies
    $$
        \reg{T}{1} \leq \left\{ 
        \begin{array}{ll}
        15 \sqrt{\keff(R)} + 4 \sqrt{5  \keff(R) T } & \\
        30 \sqrt{\keff(R)} + 20 \keffsto(R) \ln(eT) & \text{(under \Cref{condition:stochastic})} 
    \end{array} \right. 
    $$
    
\end{restatable}
The proof can be found in \Cref{app:semibandit}.
For comparison, algorithms like UCB and Tsallis-INF enjoy in the stochastic setting a regret bound of order $\sum_{a \colon \Delta(a) > 0} \ln (T) / \Delta(a)$, which is not improvable in general in the unstructured bandit problem. 
Defining $K_\Delta \coloneqq \sum_{a \colon \Delta(a) > 0} 1 / \Delta(a)$, it is easy to see that $\keffsto(R)$ improves upon $K_\Delta$, assuming as before that $\sum_{j=1}^L R_j = 1$.
In particular, Jensen's inequality gives that $\keffsto(R) \leq \sum_{j\in[L]} R_j \Gamma(j,\Delta)$, and clearly, $\Gamma(j,\Delta) \leq K_\Delta$ at any level $j$ since $|V_j| \leq K$ and every action has a unique ancestor in $V_j$.
In fact, it is usually strictly better; we demonstrate now that $\keffsto(R)$ may be exponentially smaller than $K_\Delta$ on a simple example with $\cA \subset [0,1]$ and a dyadic partitioning of the space that correlates the losses of the actions. Let $L \geq 1$ and consider for each level $V_\lvl := \big\{ i 2^{-\lvl}; i=0,\dots,2^{\lvl}-1\big\}$, $\cA := V_L$ and $K = 2^L$. For each node $v \in V_\lvl$, we set $\prt{v} = \lfloor v 2^{\lvl -1}\rfloor2^{1-\lvl}$%
and $\ell_t(v) = 0.5(v - \prt{v})+2^{-(j+1)} \epsilon_t(v)$ with $\epsilon_t(v) \sim \text{Bernoulli}(1/2)$ (we assign a loss 0 to the root). This implies that $\ell_t(v) \in [0,2^{-j}]$ and $\min_{\descof{a}{v}} \Delta(a) = v/2$.
Then, for all $j \in [L]$, 
$
   \textstyle \smash{\Gamma(j,\Delta) 
   = \sum_{i=1}^{2^j-1} \frac{2^{j+1}}{i} \leq (j+1) 2^{j+1}} \,,
$
which, using $R_j \leq 2^{-j}$, yields
$
     \textstyle \smash{\keff(R) 
    =  \brb{\sum_{j \in [L]} 2^{-\frac{j}{2}}}^2 \leq 6}
$
and
\[
   \textstyle  \keffsto(R) \coloneqq \brb{ \sum_{j \in [L]} R_j \sqrt{ \Gamma(j,\Delta) }}^2 \leq  \brb{ \sum_{j \in [L]} 2^{\frac{1-j}{2}} \sqrt{j+1} }^2\leq 6 (L+1) = 12(\log_2 K +1).
\]
On the other hand, we have $K_\Delta = \Gamma(L,\Delta) \geq 2^{L+1} L \ln 2 = 2K \ln K$. 

\section{Multi-Point Bandit Feedback}\label{sec:bandit-feedback}
Having seen that exploiting the structure of our problem is impossible under (one-point) bandit feedback, we seek to demonstrate in this section that one can skirt this obstacle by executing more than one action in each round and observing all their losses.
Precisely, just making two observations per round is already sufficient, but in some case observing more (up to $L+1$) can lead to better performance. 
Importantly, the learner does not make these observations for free, but pays the average loss of the executed actions (some of which can be duplicates).
This is decidedly milder than the feedback model considered in \Cref{sec:semi-bandit-feedback}; not least of all because the environment (its feedback mechanism in particular) is not strongly coupled to the structure of the tree.%

Nevertheless, the manner in which we exploit this limited feedback is to simulate the full-path feedback of \Cref{sec:semi-bandit-feedback} and use the same algorithmic template.
There, we leveraged the fact that at every round, there exists a loss function $\ell_t \colon \cV \rightarrow \R$ such that $\textbf{(i)}$ $\sum_{\ancof{v}{a}} \ell_t(v) = y_t(a)$ for all $a \in \cA$; $\textbf{(ii)}$ one observes $(\ell_t(v))_{\ancof{v}{A_t}}$ after executing $A_t$; and $\textbf{(iii)}$ $\ell_t(v) \leq R_j$ for all $v \in V_j$.
Here, however, we will build our own $\ell_t$ function satisfying similar properties.
Firstly, we associate with every node $v \in \overbar{\cV}$ an action $\omega_t(v)$ that is sampled from $p_t$ conditioned on having $v$ as an ancestor.
Then, define
\[
    \ell_t(v) \coloneqq y_t(\omega_t(v)) - y_t(\omega_t(\prt{v})) \,.
\]
Now, it is easy to see that Property $\mathbf{(i)}$ holds up to a constant shift (by $y_t(w_t(v_0))$), which suffices.
Moreover, an analogue of Property $\mathbf{(iii)}$ holds via the loss smoothness condition \eqref{eq:smoothness condition}; in particular, $|\ell_t(v)| \leq \scale_j$ for all $v \in V_j$.
Finally, for Property $\mathbf{(ii)}$, we sample an action $A_{t,1} \sim p_t$ and execute it along with $\omega_t(v)$ for all ancestors $v$ of $A_{t,1}$ (including $v_0$), thus we observe $(\ell_t(v))_{\ancof{v}{A_{t,1}}}$.
Note that in this manner, we execute $L+1$ actions, the marginal law of each is $p_t$, conditioned on the events of past rounds.
In practice, of course, one only needs to sample $\omega_t(\cdot)$ at the ancestors of $A_{t,1}$.

Using this reduction, we can recover the bounds of \Cref{thm:semi-bandit-feedback}, but now for $\reg{T}{L+1}$ and with $R$ swapped for $\sigma$.
Going a step further, at a slight cost, we can get away with observing $\ell_t$ at only one node, which we can achieve by executing just two actions.
In particular, we sample a level $j_t \in [L]$ from a fixed distribution $\delta$ and a node $v_t$ with probability $\propto \indicator{\cdot \in V_{j_t}}p_t[\cdot]$.
Then, we simply execute $\omega_t(v_t)$ and $\omega_t(\prt{v_t})$, giving us access to $\ell_t(v_t)$.
A more direct version of this procedure is given in \Cref{alg:2-point-feedback}.
Note that we account for the reduced feedback by correcting the loss estimate with a factor of $1/\delta_{j_t}$.
Additionally, we shift $\ell_t$ by a level-dependent constant to enforce non-negativity for technical reasons.
Ultimately, we arrive at the following theorem (see \Cref{app:bandit}).

\LinesNumbered 
\begin{algorithm}[t] 
 \KwIn{sequence of regularizers $(\psi_t)_{t \geq 1}$, sequence of constants $(b_j)_{j=1}^L$, distribution $\delta \in \Delta_{L}$}
 \KwInit{$p_1 = \argmin_{p \in \Delta_K} \psi_1(p)$}
 \For{$t \geq 1$}{
   
    sample $A_{t,1} \sim p_t\,,$
    $j_t \sim \delta\,,$ and 
    $A_{t,2} \sim p_t \mid A_{t,2} \comanc_{j_t-1} A_{t,1}$
    
    for all $a \in \cA$, set 
    $\displaystyle
        z_t(a) = \indicator{ a \comanc_{j_t} {A_{t,1}} } \frac{y_t(A_{t,1}) - y_t(A_{t,2}) + b_{j_t}} {\delta_{j_t} p_t[\anc{A_{t,1}}{j_t}]}
    $

    set $\displaystyle p_{t+1} = \argmin_{p \in \Delta_K} \ban{\summ_{s \leq t} z_s,p} + \psi_{t+1}(p)$
 }
 \caption{FTRL with 2-Point Bandit Feedback} \label{alg:2-point-feedback}
\end{algorithm}

\begin{restatable}{rethm}{thmtwopoint}\label{thm:2-point-feedback}
    \Cref{alg:2-point-feedback} with $ \psi_t(p) = 2\sqrt{6} \sum_{j \in [L]} (\scale_j/\sqrt{\delta_j})  \sqrt{\max\bcb{t, {6}/{\delta_j}}} \sum_{v \in V_j} x[v] - \sqrt{x[v]}
    $
    and $b_\lvl = \scale_\lvl$ for all $\lvl \in [L]$ satisfies \vspace*{-5pt}
    $$ 
        \reg{T}{2} \leq \left\{\begin{array}{ll} 
        18 \sqrt{\keff(\sigma/{\delta})} + 8 \sqrt{6}  \sqrt{\smash[b]{\keff(\sigma/\sqrt{\delta})} T}  \\
          36 \sqrt{\keff(\sigma/{\delta})} + 96 \keffsto(\sigma/\sqrt{\delta})\ln(eT) & \text{(under~\Cref{condition:stochastic})}
    \end{array} \right. \,.
    $$
\end{restatable}

We used here the notation of \Cref{thm:semi-bandit-feedback}, specifically, $\keff(\sigma/\sqrt{\delta}) \coloneqq \brb{\sum_{j} ({\scale_j}/{\sqrt{\delta_j}}) \textstyle \sqrt{|V_j|}}^2$
and $\keffsto(\sigma/\sqrt{\delta}) \coloneqq \brb{\sum_{j} ({\scale_j}/{\sqrt{\delta_j})} \textstyle \sqrt{\Gamma(j,\Delta)}}^2$.
The choice of $\delta$ was left open in the bounds above. 
At least for the adversarial regime, one can choose $\delta$ so as to minimize $\keff(\sigma/\sqrt{\delta})$.
In favorable cases, carefully choosing $\delta$ can render its effect on the bound negligible, see \Cref{sec:build-tree}.
In general, choosing the uniform distribution is a safe choice that results in the bound scaling with an extra factor of $\sqrt{L}$ in the adversarial regime and $L$ in the stochastic one.
If one has access to $m$-point feedback, where $2\leq m \leq L+1$, then \Cref{alg:2-point-feedback} can be generalized by randomly drawing $m-1$ levels at every round.
We show in \Cref{app:bandit} that the bounds of \Cref{thm:2-point-feedback} continue to hold in that case, with $\delta_j$ denoting the probability that level $j$ is amongst the chosen levels at any round.
Drawing the $m-1$ levels uniformly at random without replacement results in the bounds scaling with $\sqrt{L/(m-1)}$ and $L/(m-1)$ in the adversarial and the stochastic regimes respectively.

The issue of $\delta$ aside, the utility of the bounds presented here (compared to standard bandit bounds) relies on the tree being carefully constructed; specifically, that the values of $\sigma_j$ decay quickly relative to the increase in $|V_j|$.
For instance, in the example presented in \Cref{sec:semi-bandit-feedback}, $\sigma_j \leq 2^{-j+1}$, hence the bounds derived there continue to hold up to constant factors, illustrating again the huge advantage that $\keff$ and $\keffsto$ can bring.
In general, these parameters are to be used as an efficiency criteria when constructing a tree from a similarity structure.
We provide more insights on this issue in the next section.

\section{Application: Lipschitz Bandits with Two-Point Feedback}
\label{sec:build-tree}
\input{BuildingTheTree}

\acks{This work was supported by funding from the French
government, managed by the National Research Agency (ANR), under the France 2030 program, reference ANR-23-IACL-0006. Additionally, this work was supported in part by ANR in the framework of the PEPR IA FOUNDRY project (ANR-23-PEIA-0003). 
PM is also a member of the Archimedes Research Unit/Athena RC, and was partially supported by project MIS 5154714 of the National Recovery and Resilience Plan Greece 2.0 funded by the European Union under the NextGenerationEU Program.}

\bibliography{colt2026}

\newpage

\appendix
\crefalias{section}{appendix} %
\crefalias{subsection}{appendix} %

\section{Motivation: Similarity structures and related notions}
\label{app:similarities}
\input{App-Similarities}

\section{Proofs of \texorpdfstring{\Cref{sec:lowerbound}}{Section~\ref{sec:lowerbound}}}\label{app:lowerbound}
We provide here the proof of \Cref{prop:one-point-tree-lb}, which is restated below.
\proplowerbound*
\begin{proof}
Let $\rho\coloneqq \scale_L$. We define

\[
C_\rho \coloneqq \Big\{ x\in[0,1]^K : \max_{a,b\in\cA}\abs{x(a)-x(b)} \le \rho \Big\}
\qquad \text{and} \qquad
C_\rho^T \coloneqq \Big\{ (y_t)_{t=1}^T : \forall t,\ y_t\in C_\rho \Big\}.
\]

By \citet[Corollary~4]{gerchinovitz2016refined}, for $K>2$, $T>32(K-1)\log(14)$, and
$\rho > 0.22\sqrt{(K-1)/T}$, any randomized one-point bandit algorithm satisfies
\[
\sup_{(y_t)_{t=1}^T \in C_{\rho}^T}\ \E\!\left[\reg{T}{1}(y_{1:T})\right]
\;>\; \frac{1}{504}\sqrt{T(K-1)}.
\]
It remains to show that $C_{\rho}^T \subseteq \cY^T_\scale$.
Fix any $x\in C_{\rho}$ and any $a,b\in\cA$. Then $\abs{x(a)-x(b)}\le \rho=\scale_L$.
Since $s(a,b)\in[L]$ for $a\neq b$ and $(\scale_j)$ is nonincreasing, we have $\scale_L \le \scale_{s(a,b)}$.
Thus $\abs{x(a)-x(b)} \le \scale_{s(a,b)}$, so $x$ satisfies the tree-compatibility constraint.
Therefore $C_{\rho}^T \subseteq \cY^T_\scale$, and taking the supremum over $\cY^T_\scale$ yields the proposition.
\end{proof}

\section{Proofs of \texorpdfstring{\Cref{sec:semi-bandit-feedback}}{Section~\ref{sec:semi-bandit-feedback}}}\label{app:semibandit}
\LinesNumbered 
\begin{algorithm}[t] 
 \KwIn{sequence of regularizers $(\psi_t)_{t \geq 1}$}
 \KwInit{$p_1 = \argmin_{p \in \Delta_K} \psi_1(p)$}
 \For{$t \geq 1$}{
    sample $A_{t} \sim p_t$ and set 
    $
        z_t(a) = \sum_{j \in [L]} \indicator{a \comanc_{j} A_t} \frac{\ell_t(\anc{A_t}{j})} {p_t[\anc{A_t}{j}]}
    $
    for all $a \in \cA$
    
    set $\displaystyle p_{t+1} = \argmin_{p \in \Delta_K} \ban{\summ_{s \leq t} z_s,p} + \psi_{t+1}(p)$
 }
 \caption{FTRL with Semi-Bandit Feedback} \label{alg:semi-bandit-feedback}
\end{algorithm}

For convenience, we summarize the algorithm template used in \Cref{sec:semi-bandit-feedback} in \Cref{alg:semi-bandit-feedback}.
We prove the following proposition before moving on to the proof of \Cref{thm:semi-bandit-feedback}, which is restated afterwards.
\begin{restatable}{reprop}{propsemi}\label{prop:semi-bandit-feedback}
    \Cref{alg:semi-bandit-feedback} with %
    $\displaystyle
        \psi_t(p) = \sqrt{ 5 \max\{t, 20\} } \sum_{j \in [L]} R_j \sum_{v \in V_j} p[v] - \sqrt{p[v]}  
    $
    satisfies that
    \[ \displaystyle
        \reg{T}{1} \leq 15 \sum_{j \in [L]} {R}_j \textstyle \sqrt{|V_j|} \displaystyle + 2 \sqrt{5} \sum_{t=1}^{T} \frac{1}{\sqrt{t}} \E \sum_{j \in [L]} {R}_j \sum_{v \in V_j}  \sqrt{p_t[v]} - p_t[v] \,. 
    \] 
\end{restatable}
\begin{proof} 
    This result is an application of \Cref{lem:ftrl-nested-tsallis-sparse} with $c=5$; $\xi_j = 1$ for all $j \in [L]$; $M_t = [L]$ for all $t \geq 1$; and filtration $(\cF_t)_{t \geq 0}$ with $\cF_t$ being the $\sigma$-algebra generated by $(A_s)_{s \leq t}$ and $(\ell_s)_{s \leq t}$.
    In particular, as \Cref{lem:ftrl-nested-tsallis-sparse} requires, it holds that $\ell_t(v) \in [0, R_j]$ for all $j\in[L]$, $v \in V_j$, and $t \geq 1$; $p_t$ is measurable \wrt $\cF_{t-1}$; the law of $A_t$ conditioned on $\cF_{t-1}$ is $p_t$; and for all $a,a' \in \cA$,
    \[\sum_{\ancof{v}{a}} \E_t[\ell_t(v) \mid \ancof{v}{A_t}] - \sum_{\ancof{w}{a'}} \E_t[\ell_t(w) \mid \ancof{w}{A_t}] = \E_t[y_t(a) - y_t(a')]\]
    simply since $\sum_{\ancof{v}{a}} \ell_t(v) = y_t$ and $\ell_t(v)$ is independent of $\indicator{\ancof{v}{A_t}}$ conditioned on $\cF_{t-1}$ as the adversary can only react to the learner's decisions that were made up to the previous round.
\end{proof}

\thmsemi*
\begin{proof}
    Starting from the bound of \Cref{prop:semi-bandit-feedback}, one can use the facts that $\sum_{v \in V_j}  \sqrt{p_t[v]} - p_t[v] \leq \sqrt{|V_j|}$ for any $j \in [L]$ and that $\sum_{t=1}^T 1/\sqrt{t} \leq 2 \sqrt{T}$
    to directly obtain the first result.
    Concerning the second result, \Cref{lem:variational-bound-tsallis} implies that 
    \begin{align*}
        2 \sum_{j \in [L]} \frac{2 \sqrt{5} R_j}{\sqrt{t}} \sum_{v \in V_j}  \sqrt{p_t[v]} - p_t[v]   \leq \inprod{\Delta,p_t} + \frac{20}{t} \sum_{j \in [L]} {R_j^2} \sum_{v \in V_j \colon \tilde{\Delta}(v) \neq 0} \frac{1}{ \tilde{\Delta}(v)} 
    \end{align*}
    for any $\tilde{\Delta} \colon \cV \rightarrow \R$ such that $(i)$ $\tilde{\Delta}(v) \geq 0$ $\forall v \in V$, $(ii)$ $\Delta(a) \geq \sum_{\ancof{v}{a}} \tilde{\Delta}(v)$ $\forall a \in \cA$, and $(iii)$ $|\{v \in V_j \mid \tilde{\Delta}(v) = 0\}| \leq 1$ for all $j \in [L]$.
    Then, using that $\E\Delta(A_t) = \E \inprod{\Delta,p_t}$, we obtain that
    \begin{align*}
        \reg{T}{1} \leq  15 \sum_{j \in [L]} {R}_j \sqrt{|V_j|} + \frac{1}{2} \E \lsb{ \sum_{t=1}^T \Delta(A_t) } + 10 \sum_{t=1}^T \frac{1}{t} \sum_{j \in [L]} {R_j^2} \sum_{v \in V_j \colon \tilde{\Delta}(v) \neq 0} \frac{1}{ \tilde{\Delta}(v)} \,.
    \end{align*}
    Then, using that $\sum_{t=1}^T 1/t \leq \ln(eT)$ and that, by assumption, the second term is at most $\reg{T}{1} / 2$, we get that
    \begin{align} \label{eq:stochastic-bound-general}
        \reg{T}{1} \leq  30 \sum_{j \in [L]} {R}_j \sqrt{|V_j|} + 20 \sum_{j \in [L]} {R_j^2} \sum_{v \in V_j \colon \tilde{\Delta}(v) \neq 0} \frac{\ln(e T)}{ \tilde{\Delta}(v)} \,.
    \end{align}
    To reach a more concrete bound we impose now a simple form on $\tilde{\Delta}$.
    In particular, let $\hat{\Delta}(v) \coloneqq \min_{\descof{a}{v}} \Delta(a)$, and fixing any distribution $q \in \Delta_{L}$ (with strictly positive weights) over the levels of the tree, we take $\tilde{\Delta}(v) = q(j) \hat{\Delta}(v)$ for $v \in V_j$.
    This is a valid choice since it is non-negative, satisfies
    \[
        \sum_{\ancof{v}{a}} \tilde{\Delta}(v) = \sum_{j \in [L]} q(j) \hat{\Delta}(\anc{a}{j}) \leq \sum_{j \in [L]} q(j) {\Delta}(a) = \Delta(a) 
    \]
    for any $a \in \cA$, and can only vanish at one node per level by the assumption that $|\{a\in \cA \colon \Delta(a)=0\}| \leq 1$.
    It then follows from \eqref{eq:stochastic-bound-general} that
    \begin{align*}
        \reg{T}{1} \leq  30 \sum_{j \in [L]} {R}_j \sqrt{|V_j|} + 20 \sum_{j \in [L]} \frac{R_j^2}{q(j)} \sum_{v \in V_j \colon \hat{\Delta}(v) \neq 0} \frac{\ln(e T)}{ \hat{\Delta}(v)} \,. 
    \end{align*}
    The sought bound now follows by choosing $q(j) \propto R_j \sqrt{\sum_{v \in V_j \colon \hat{\Delta}(v) \neq 0} 1/{ \hat{\Delta}(v)}}$.
    
\end{proof}

\section{Proofs of \texorpdfstring{\Cref{sec:bandit-feedback}}{Section~\ref{sec:bandit-feedback}}}\label{app:bandit}

\LinesNumbered 
\DontPrintSemicolon
\begin{algorithm}[t] 
 \KwIn{sequence of regularizers $(\psi_t)_{t \geq 1}$, sequence of constants $(b_j)_{j=1}^L$, number of extra queried points $n \in [L]$, distribution $\dist$ over subsets of $[L]$ of cardinality $n$}
 \KwDef{for $j \in [L]$,  let $\xi_j \coloneqq \sum_{S \in \supp(\dist)} \indicator{j \in S} \dist(S)$}
 \KwInit{$p_1 = \argmin_{p \in \Delta_K} \psi_1(p)$}
 \For{$t \geq 1$}{

    sample $A_{t,1} \sim p_t$

    sample $M_t \sim 
    \dist$

    set $M_t^\shortdownarrow = \{j \in M_t \mid j+1 \notin M_t\}$ %

    \lFor{$j \in M_t^\shortdownarrow$}{set $\smp{t}{j} = A_{t,1}$}

    $i \gets 2$

    \For{$j \in M_t$}{        
        sample $A_{t,i} \sim p_t \mid A_{t,i} \comanc_{j-1} A_{t,1}$ 

        set $\smp{t}{j-1} = A_{t,i}$

        $i \gets i + 1$
    }
        
    for all $a \in \cA$, set 
    $\displaystyle
        z_t(a) = \sum_{j \in M_t} \indicator{ a \comanc_{j} {A_{t,1}}} \frac{y_t(\smp{t}{j}) - y_t(\smp{t}{j-1}) + b_j} {\xi_j \,p_t[\anc{A_{t,1}}{j}]}
    $

    set $\displaystyle p_{t+1} = \argmin_{p \in \Delta_K} \ban{\summ_{s \leq t} z_s,p} + \psi_{t+1}(p)$
 }
 \caption{FTRL with $n+1$-Point Bandit Feedback} \label{alg:n-point-feedback}
\end{algorithm}

We provide here a generalization of \Cref{alg:2-point-feedback} to handle $n+1$-point feedback, with $n$ between $2$ and $L$.
We prove a generic regret bound for this algorithm, from which \Cref{thm:2-point-feedback} (restated afterwards) follows as a corollary.

\begin{theorem}\label{thm:n-point-feedback}
    \Cref{alg:n-point-feedback} with $\psi_t \colon \R_+^d \rightarrow \R$ given by
    \begin{align*}
        \psi_t(x) \coloneqq \frac{4 \sqrt{c}}{\sqrt{(c-1)}} \sum_{j \in [L]} \frac{\scale_j}{\sqrt{\xi_j}} \max\Bcb{ \sqrt{t}, \sqrt{{c(c-1)}/{\xi_j} } } \sum_{v \in V_j} x[v] - \sqrt{x[v]}
    \end{align*}
    where $c \geq 2$, 
    and $b_\lvl = \scale_\lvl$ for all $\lvl \in [L]$ satisfies 
    \begin{align*}
        \reg{T}{n+1} \leq 6c \sum_{j \in [L]} \frac{\scale_j}{ \xi_j} \textstyle \sqrt{|V_j|} \displaystyle + \frac{16 \sqrt{c}}{\sqrt{(c-1)}} \sum_{j \in [L]} \frac{\scale_j}{\sqrt{\xi_j}} \textstyle \sqrt{|V_j| T} \displaystyle \,.
    \end{align*}
    Further, 
    if \Cref{condition:stochastic} holds,
    then, \Cref{alg:n-point-feedback} also satisfies
    \begin{align*}
        \reg{T}{n+1} \leq  12 c \sum_{j \in [L]} \frac{\scale_j}{ \xi_j} \textstyle \sqrt{|V_j|} \displaystyle + \frac{64 c}{c-1} \Bbrb{ \sum_{j \in [L]} \frac{\scale_j}{\sqrt{\xi_j}} \sqrt{ \Gamma(j,\Delta) }}^2 \ln(eT) \,.
    \end{align*}
\end{theorem}
\begin{proof}  
    Let $\cF_{t}$ be the $\sigma$-algebra generated by $(A_{s,i})_{s \leq {t},i\in [n+1]}\,$, $(M_s)_{s \leq {t}}\,$, and $(y_s)_{s \leq t}\,$.
    In the sequel, we will for brevity denote $\pr(\cdot \mid \cF_{t-1})$ by $\pr_t(\cdot)$ and $\E[\cdot \mid \cF_{t-1}]$ by $\E_t[\cdot]$.
    Clearly, $\pr_t(A_{t,1} = a) = p_t(a)$.
    For $2 \leq i \leq n+1$, we posit that a fixed function $\zeta_i \colon \supp(\dist) \rightarrow \{0,\dots,L-1\}$ gives, at every round $t$, the unique level $\zeta_i(M_t)$ such that $A_{t,i} = \smp{t}{\zeta_i(M_t)}$.
    Notice then that for $2 \leq i \leq n+1$,
    \begin{align*}
         \pr_t(A_{t,i} = a)
         &= \sum_{a' \in \cA, S \in \supp(\dist)} \pr_t \brb{  A_{t,i} = a \mid A_{t,1} = a', M_t = S}  \pr_t \brb{A_{t,1} = a', M_t = S}   \\
         &= \sum_{a' \in \cA, S \in \supp(\dist)} \pr_t \brb{ A_{t,i} = a \mid A_{t,1} = a', M_t = S}  \pr_t \brb{A_{t,1} = a'} \pr_t \brb{M_t = S}  \\
         &= \sum_{a' \in \cA, S \in \supp(\dist)} \indicator{a' \comanc_{\zeta_i(S)} a} \frac{p_t(a)}{p_t[\anc{a'}{\zeta_i(S)}]}  p_t(a') \dist(S)   \\
         &= p_t(a) \sum_{S \in \supp(\dist)} \dist(S) \sum_{a' \in \cA} \indicator{a' \comanc_{\zeta_i(S)} a} \frac{1}{p_t[\anc{a'}{\zeta_i(S)}]}  p_t(a')   \\
         &= p_t(a) \sum_{S \in \supp(\dist)} \dist(S)  \frac{1}{p_t[\anc{a}{\zeta_i(S)}]}  \sum_{a' \in \cA} \indicator{a' \comanc_{\zeta_i(S)} a}  p_t(a')   = p_t(a) \,,
    \end{align*}
    where the second equality uses that $A_t$ and $S_t$ are independent conditioned on $\cF_{t-1}$, and the third equality follows from the sampling rules of \Cref{alg:n-point-feedback}.
    Combining this with the fact that $y_t$ is independent of the learner's actions at round $t$ conditioned on $\cF_{t-1}$, we get that
    \begin{align*}
        \E \lsb{\frac{1}{n+1} \sum_{i=1}^{n+1} y_t(A_{t,i})} = \E \bsb{ y_t(A_{t,1}) } \,.
    \end{align*}
    Thus,
    \begin{align} \label{eq:n-point-regret-to-one-point}
        \reg{T}{n+1} = \E \lsb{\sum_{t=1}^T y_t(A_{t,1})} - \min_{a \in \cA} \E \lsb{ \sum_{t=1}^T y_t (a)} \,.
    \end{align} 
    Towards applying \Cref{lem:ftrl-nested-tsallis-sparse}, fix a mapping $\lambda \colon \overbar{\cV} \rightarrow \cA$ such that $\ancof{v}{\lambda(v)}$ for all $v \in \overbar{\cV}$.
    Further, for each $v \in \overbar{\cV}$, let $({\omega}_t(v))_{t \geq 1}$
    be a random sequence of actions given by
    \[
        {\omega}_t(v) = \begin{cases}
            \smp{t}{\level{v}} \quad &\text{if} \: \ancof{v}{A_{t,1}} \text{ and } \{\level{v},\level{v}+1\} \cap M_t \neq \emptyset \\
            \lambda(v) &\text{otherwise}\,.
        \end{cases}
    \] 
    Note that for $a \in \cA$, $\omega_t(a) = a$.
    Next, we define at every round $t$ a mapping $\ell_t \colon \cV \rightarrow \R$ given by
    \[
        \ell_t(v) \coloneqq y_t({\omega}_t(v)) - y_t({\omega}_t(\prt{v})) + \scale_{\level{v}}\,.
    \]
    Note that for any $v \in V_j$,
    \[
        \abs{\ell_t(v)} \leq 2 \scale_j \quad \text{and} \quad \ell_t(v) \geq 0 \,.
    \]
    Further, define 
    \[
        \tilde{\ell}_t(v) \coloneqq \E_t \bsb{ {\ell}_t(v)\mid \ancof{v}{A_{t,1}}, \level{v} \in M_t } \,,
    \]
    so that
    \begin{align*}
        \tilde{\ell}_t(v)
        &= \E_t \bsb{ y_t({\omega}_t(v))\mid \ancof{v}{A_{t,1}}, \level{v} \in M_t } \\
        &\hspace{10em}- \E_t \bsb{ y_t({\omega}_t(\prt{v})) \mid \ancof{v}{A_{t,1}}, \level{v} \in M_t } + \scale_{\level{v}} \,.
    \end{align*}
    We aim to show now that the sum of the first two terms over $v \in \cV \colon \ancof{v}{a}$ is a telescoping sum for any action $a \in \cA$.
    Towards that, we can rewrite the second term  as follows for any $j \in [L]$ and $v \in V_j$:
    \begin{align*}
        \E_t \bsb{ y_t({\omega}_t(\prt{v})) \mid \ancof{v}{A_{t,1}}, \level{v} \in M_t} &= \E_t \bsb{ y_t({\omega}_t(\prt{v})) \mid v = \anc{A_{t,1}}{j}, j \in M_t } \\
        &= \E_t \bsb{ y_t(\smp{t}{j-1}) \mid v = \anc{A_{t,1}}{j}, j \in M_t } \\
        &= \E_t \Bsb{ \E_t \bsb{  y_t(\smp{t}{j-1}) \mid v = \anc{A_{t,1}}{j}, j \in M_t , y_t} } \\
        &= \E_t \Bsb{ \summ_{a \in \cA} \indicator{\ancof{\prt{v}}{a}} \frac{p_t(a)}{p_t[\prt{v}]} y_t(a)} \,,
    \end{align*}
    where the second equality follows from the definition of $\omega_t$, the third is an application of the tower rule; and the fourth follows from the sampling rules of \Cref{alg:n-point-feedback}, the fact that $\prt{v} = \anc{A_{t,1}}{j-1}$ whenever $v = \anc{A_{t,1}}{j}$, and the fact that $y_t$ is drawn by the environment independently of the learner's actions at round $t$, conditioned on $\cF_{t-1}$.
    At the same time, it similarly holds that
    \begin{align*}
        \E_t \bsb{ y_t({\omega}_t({v}))\mid \ancof{v}{A_{t,1}}, \level{v} \in M_t }
        &= \E_t \bsb{ y_t({\omega}_t({v}))\mid v = \anc{A_{t,1}}{j}, j \in M_t} \\
        &= \E_t \bsb{ y_t(\smp{t}{j})\mid v = \anc{A_{t,1}}{j}, j \in M_t} \\
        &= \E_t \Bsb{ \E_t \bsb{ y_t(\smp{t}{j}) \mid v = \anc{A_{t,1}}{j}, j \in M_t, y_t  }} \\
        &= \E_t \Bsb{ \summ_{a \in \cA} \indicator{\ancof{{v}}{a}} \frac{p_t(a)}{p_t[{v}]} y_t(a)} \,,
    \end{align*}
    where in the last step we used that, at round $t$, given that $\anc{A_{t,1}}{j}$ and $j \in M_t$, the conditional law (\wrt $\cF_{t-1}$) of $\smp{t}{j}$ is $\indicator{\ancof{v}{\cdot}} p_t(\cdot)/p_t[v]$ regardless of whether $j$ belongs to $M_t^\downarrow$ or $M_t \setminus M_t^\downarrow$; in the latter case, this holds directly via the sampling rule of \Cref{alg:n-point-feedback} for $\smp{t}{j}$, while in the former, we use that $\smp{t}{j} = A_{t,1}$ and that conditioned on $\cF_{t-1}$, the law of $A_{t,1}$ is $p_t$.
    We can then conclude that for any $a \in \cA$,
    \begin{align*}
        \sum_{\ancof{v}{a}} \tilde{\ell}_t(v) =  \E_t\bsb{y_t(a)} -  \sum_{a' \in \cA} p_t(a') \E_t \bsb{y_t(a')} + \sum_{j \in [L]} \scale_j\,.
    \end{align*}
    Implying that for any $a, a' \in \cA$,
    \[
        \sum_{\ancof{v}{a}} \tilde{\ell}_t(v)- \sum_{\ancof{w}{a'}} \tilde{\ell}_t(w) = \E_t\bsb{y_t(a) - y_t(a')} \,.
    \]
    Notice, moreover, that
    \begin{align*}
        &\sum_{\ancof{v}{a}} \indicator{ \ancof{v}{A_{t,1}}, \level{v} \in M_t } \frac{\ell_t(v)} {\xi_{\level{v}} p_t[v]} \\
        &\qquad= \sum_{\ancof{v}{a}} \indicator{ \ancof{v}{A_{t,1}}, \level{v} \in M_t } \frac{y_t(\omega_t(v)) - y_t(\omega_t(\prt{v}))  + \scale_{\level{v}}} {\xi_{\level{v}} p_t[v]} \\
        &\qquad= \sum_{j \in [L]} \indicator{ a \comanc_{j} {A_{t,1}}, j \in M_t } \frac{y_t\brb{\omega_t(\anc{a}{j})} - y_t\brb{\omega_t(\anc{a}{j-1})} + \scale_{j}} {\xi_j \, p_t[\anc{a}{j}]} \\
        &\qquad= \sum_{j \in [L]} \indicator{ a \comanc_{j} {A_{t,1}}, j \in M_t } \frac{y_t\brb{\omega_t(\anc{A_{t,1}}{j})} - y_t\brb{\omega_t(\anc{A_{t,1}}{j-1})} + \scale_{j}} {\xi_j \, p_t[\anc{A_{t,1}}{j}]} \\
        &\qquad= \sum_{j \in M_t } \indicator{ a \comanc_{j} {A_{t,1}}} \frac{y_t\brb{\omega_t(\anc{A_{t,1}}{j})} - y_t\brb{\omega_t(\anc{A_{t,1}}{j-1})} + \scale_{j}} {\xi_j \, p_t[\anc{A_{t,1}}{j}]} \\
        &\qquad=  \sum_{j \in M_t} \indicator{ a \comanc_{j} {A_{t,1}}} \frac{y_t(\smp{t}{j}) - y_t(\smp{t}{j-1}) + \scale_j} {\xi_j \,p_t[\anc{A_{t,1}}{j}]} = z_t(a) \,, 
    \end{align*}
    where the last step follows form the definition of $\omega_t$.
    Finally, it is easily verifiable that in our case, $\pr_t(j \in M_t) = \xi_j\:\:\forall j \in [L]$, and $\indicator{\ancof{v}{A_t}} \ind \indicator{\level{v} \in M_t} \mid \cF_{t-1}\:\: \forall v \in \cV$.
    We can now invoke \Cref{lem:ftrl-nested-tsallis-sparse} with $R_j = 2 \scale_j$ for $j \in [L]$; $A_t = A_{t,1}$ for $t \geq 1$; and $\xi_j$, $M_t$, $y_t$, $\ell_t$, and $(\cF_{t})_{t \geq 0}$ as used in the current context; yielding that
    \begin{multline*}
        \E \lsb{\sum_{t=1}^T y_t (A_{t,1})} - \min_{a \in \cA} \E \lsb{ \sum_{t=1}^T y_t (a)} 
        \leq \frac{8 \sqrt{c}}{\sqrt{(c-1)}} \sum_{t=1}^{T} \frac{1}{\sqrt{t}} \sum_{j \in [L]} \frac{\scale_j}{\sqrt{\xi_j}} \E \sum_{v \in V_j}  \brb{\sqrt{p_t[v]} - p_t[v]} \\
        + 6c \sum_{j \in [L]} \frac{\scale_j}{\xi_j} \sqrt{|V_j|} \,.
    \end{multline*}
    which is also a bound on $\reg{T}{n+1}$ via \eqref{eq:n-point-regret-to-one-point}. 
    At this junction, the rest of the proof proceeds similarly to the proof of \Cref{thm:semi-bandit-feedback}. 
\end{proof}

\thmtwopoint*
\begin{proof}
    This follows as a special case of \Cref{thm:n-point-feedback} choosing $c=3$ and $M_t = \{j_t\}$ at every round.    
\end{proof}

\section{Proofs of \texorpdfstring{\Cref{sec:build-tree}}{Section~\ref{sec:build-tree}}}\label{app:lipschitz}

\subsection{Proof of \texorpdfstring{\Cref{cor:corgeneral}}{Corollary~\ref{cor:corgeneral}}}
We provide here the proof of \Cref{cor:corgeneral}, which is restated below.

\corgeneral*

\begin{proof}
Firstly, since $\actions_\nlvls$ is a $\treecst 2^{-\nlvls}$-cover of $\actions$, we have that
\begin{align*}
    \reg{T}{2} \leq \E\left[\sum_{t=1}^T \frac{1}{2}\brb{y_t(A_{t, 1})+y_t(A_{t, 2})}\right]
-
\min_{a\in\actions_\nlvls}\ \E\left[\sum_{t=1}^T y_t(a)\right] + \treecst \, 2^{-\nlvls} T \,.
\end{align*}
So, using \Cref{thm:2-point-feedback} gives that
\begin{align*}
    \reg{T}{2} &\leq 
        18 \sqrt{\keff(\sigma/{\delta})} + 8 \sqrt{6}  \sqrt{{\keff(\sigma/\sqrt{\delta})} T} +  \treecst \, 2^{-\nlvls} T \\
        &\leq 2 (9 + 4 \sqrt{6}) \sqrt{{\keff(\sigma/\sqrt{\delta})} T} +  \treecst \, 2^{-\nlvls} T \,,
\end{align*}  
where we have used in the second step that $\sqrt{\keff(\sigma/{\delta})} \leq \sqrt{{\keff(\sigma/\sqrt{\delta})} T}$ by the assumption that $\min_{\lvl \in [\nlvls]}\delta_\lvl \geq 1/T $.
Next, observe that
\begin{align*}
    \sqrt{{\keff(\sigma/\sqrt{\delta})}} = \sum_{\lvl \in [\nlvls]} \frac{\scale_\lvl}{\sqrt{\delta_\lvl}} \sqrt{|\prunactions_\lvl|} = \treecst 2^{3} \sum_{\lvl \in [\nlvls]} \frac{2^{-\lvl}}{\sqrt{\delta_\lvl}} \sqrt{|\prunactions_\lvl|} \leq \treecst 2^{3} \sum_{\lvl \in [\nlvls]} \frac{2^{-\lvl}}{\sqrt{\delta_\lvl}} \sqrt{|\actions_\lvl|} \,,
\end{align*}
where the last step uses that $\prunactions_\lvl \subseteq \actions_\lvl$.
Finally, our choice for $\delta_\lvl$ entails that
\[
    \sum_{\lvl \in [\nlvls]} \frac{2^{-\lvl}}{\sqrt{\delta_\lvl}} \sqrt{|\actions_\lvl|} = \lrb{\sum_{\lvl \in [\nlvls]} 2^{-2\lvl/3} {|\actions_\lvl|}^{1/3}}^{3/2} \,.
\]
\end{proof}

\subsection{On the Near-Optimality of the Tree Constructed in \texorpdfstring{\Cref{sec:build-tree}}{Section~\ref{sec:build-tree}}}\label{app:tree-near-optimal}
We will assume here that $\actions$ is finite (with cardinality $K$) and that $T$ is large enough so that $\dmin \coloneqq \min_{a,b \in \actions \colon a \neq b} d_{\actions,\infty}(a,b) \geq \sqrt{K/T}$.
The effect of this is that executing \Cref{alg:2-point-feedback} on a strict subset of the action set would force a worst-case approximation error of order $\sqrt{K T}$, which would nullify any attempt to leverage the structure.
Now, any nested sequence $(\cP_\lvl)_{\lvl=0}^{\nlvls'}$ of partitions of $\actions$ such that $\cP_0 = \{\actions\}$ and $\cP_{\nlvls'} = \{\{a\}\}_{a \in \actions}$ can be used to build a tree of $\nlvls'$ levels with its leaves coinciding with the action set.
In a similar manner to \Cref{cor:corgeneral}, executing \Cref{alg:2-point-feedback} on this tree results in a regret bound featuring a $\keff$ of (with the optimal choice of $\delta$) 
\[
    \Bbrb{ \sum_{\lvl=1}^{\nlvls'} \scale_{*,\lvl}^{2/3} {| \cP_{\lvl} |}^{1/3} }^3 \quad \text{where} \: \sigma_{*,\lvl} \coloneqq \max_{S \in \cP_{\lvl-1}} \max_{a,b \in S} d_{\actions, \infty}(a,b) \,.
\]
It is not clear if directly optimizing this quantity in terms of the nested structure is a feasible endeavor.
Instead, we show in the following lemma that the recipe we described for constructing the tree---by forming a $\treecst 2^{-\lvl}$ cover of $\actions$ at level $\lvl$---minimizes this quantity up to small constants, provided that these covers are optimal.
Fortunately, near-optimal covers can be constructed in time polynomial in $K$.
At any level $\lvl$, we can associate each action $a \in \cA$ with the set $\{ a' \in \cA \colon d_{\actions,\infty}(a,a') \leq  \treecst 2^{-\lvl}\}$; so, building an optimal $\treecst 2^{-\lvl}$-cover reduces to the set cover problem, which can be solved in polynomial time using a simple greedy algorithm \citep{johnson1974}.
In particular, using this greedy algorithm, one builds a cover that is at worst an extra factor of $1 + \ln K $ larger than the minimum size \citep{johnson1974,slavik1997}.
\begin{lemma} \label{lem:tree-optimality}
    Set $\treecst = \max_{a,b \in \actions} d_{\actions,\infty} (a,b)$ and $\nlvls = 1 + \floor{-\log_2 (\treecst^{-1} \dmin)}$.
    Let $\actions_0 \coloneqq \{a_0\}$ for some arbitrary action $a_0$ and $(\actions_\lvl)_{\lvl\in [\nlvls]}$ be a sequence of subsets of $\actions$ such that $\actions_\lvl$ is an optimal $\treecst 2^{-\lvl}$-cover of $\actions$. 
    Then, for any nested sequence $(\cP_i)_{i=0}^{\nlvls'}$ of partitions of $\actions$ such that $\cP_0 = \{\actions\}$ and $\cP_{\nlvls'} = \{\{a\}\}_{a \in \actions}$,
    it holds that
    \[
        \treecst^2 \lrb{\sum_{\lvl=1}^{\nlvls} 2^{-2\lvl/3} {|\actions_\lvl|}^{1/3}}^3 \leq 20 \Bbrb{ \sum_{i=1}^{\nlvls'} \scale_{*,i}^{2/3} {| \cP_{i} |}^{1/3} }^3 \,,
    \]
    where $\sigma_{*,i} \coloneqq \max_{S \in \cP_{i-1}} \max_{a,b \in S} d_{\actions, \infty}(a,b)$.
\end{lemma}
Note that this choice of $\nlvls$ ensures that $\treecst 2^{-\nlvls} < \dmin$, which implies that $\actions_\nlvls = \actions$.
\begin{proof}
    We begin by defining a mapping $f \colon [\nlvls] \rightarrow [\nlvls']$ as follows:
    \begin{equation} \label{def:tree-to-tree}
        f(\lvl) \coloneqq \min \Bcb{ i \in [\nlvls'] \mid \max_{a,b \in S} d_{\actions,\infty}(a,b) < \treecst 2^{-\lvl} \:\forall S \in \cP_i } \,.
    \end{equation}
    Since $\cP_{\nlvls'} = \{\{a\}\}_{a \in \actions}$, $\max_{a,b \in S} d_{\actions,\infty}(a,b) = 0$ for all $S \in \cP_{\nlvls'}$; thus, the set in the R.H.S. of \eqref{def:tree-to-tree} is never empty.
    Note that the pre-image of $i \in [\nlvls']$, if non-empty, is an interval $B_i \subseteq [\nlvls]$ satisfying that
    \[
        \max_{\lvl \in B_i} \treecst 2^{-\lvl}\leq \scale_{*,i} 
        \,.
    \]
    This is easy to see for $i=1$: since we set $\treecst = \max_{a,b \in \actions} d_{\actions,\infty} (a,b)$, we have that $\scale_{*,1} = \treecst$.
    (Recall that $\scale_{*,i}$ concerns the maximum diameter of the subsets of $\cP_{i-1}$, not $\cP_{i}$.)
    For $i \geq 2$, if $\treecst 2^{-\lvl} > \scale_{*,i} $ for a certain $\lvl \in B_i$, then it must be that $\lvl \in B_{i'}$ for some $i' \in \{1,\dots,i-1\}$, which is a contradiction.
    It then follows that
    \[
        \sum_{\lvl \in B_i} \brb{c2^{-\lvl}}^{2/3}  %
        \leq  \scale^{2/3}_{*,i}  \sum_{r=0}^\infty 2^{-2r/3} \leq \frac{2^{2/3}}{2^{2/3}-1}\scale^{2/3}_{*,i}  \,.
    \]
    This allows concluding that
    \begin{equation*}
        \treecst^2 \lrb{\sum_{\lvl=1}^{\nlvls} 2^{-2\lvl/3} {|\cP_{f(\lvl)}|}^{1/3}}^3 \leq 20 \Bbrb{ \sum_{i=1}^{\nlvls'} \scale_{*,i}^{2/3} {| \cP_{i} |}^{1/3} }^3 \,.
    \end{equation*}
    To obtain the sought result, it remains to use that $|\cP_{f(\lvl)}| \geq |\actions_\lvl|$.
    This holds since $\actions_\lvl$ is a $\treecst 2^{-\lvl}$-cover of $\actions$ with minimum cardinality, meantime, collecting an arbitrary member of each subset in $\cP_{f(\lvl)}$ results too in a $\treecst 2^{-\lvl}$-cover of $\actions$.
\end{proof}

\subsection{Proof of \texorpdfstring{\Cref{cor:lipschitznew}}{Corollary~\ref{cor:lipschitznew}}}
We provide here the proof of \Cref{cor:lipschitznew}, which is restated below.

\corlipschitznew*

\begin{proof}
    Since $|\actions_\lvl| \leq 2^{\lvl d}$ and $\nlvls = \floor{\frac{1}{d}\log_2(T) }$, \Cref{cor:corgeneral} gives that (supposing for now that $\min_\lvl \delta_\lvl \geq 1/T$)
    \[
        \reg{T}{2} \leq 
        16 \cdot (9 + 4 \sqrt{6}) \cdot GD\sqrt{d} \cdot \Bbrb{\sum_{\lvl \in [\nlvls]} 2^{\lvl (d-2)/3}}^{3/2} \sqrt{  T} + 2 GD\sqrt{d} \, T^{\frac{d-1}{d}} \,.
    \]
    Then, observe that
    \[
    	\Bbrb{\sum_{\lvl \in [\nlvls]} 2^{\lvl (d-2)/3}}^{3/2} \leq \left\{ \begin{array}{ll} 
    		L^{3/2} & \text{if } d = 2 \\
    		11 \cdot 2^{L(\frac{d}{2}-1)_+} & \text{if } d \neq 2
    		\end{array} \right.
    \]
where $(\,\cdot\,)_+ = \max\{\cdot, 0\}$.
This leads to the stated regret bound after using that $\nlvls = \floor{\frac{1}{d}\log_2(T) }$.

\medskip \noindent
To conclude the proof, it only remains to check that $\max_{\lvl \in [\nlvls]}\delta_\lvl^{-1/2} \leq \sqrt{T}$ for $T \geq 8$. Indeed, 
\begin{itemize}[nosep]
	\item if $d =1$, 
	\[
		\delta_\lvl^{-1/2} \leq \delta_\nlvls^{-1/2} = 2^{\frac{L}{6}} \bigg( \sum_{\lvl=1}^L 2^{-\frac{\lvl}{3}} \bigg)^{1/2} \leq 2 \cdot 2^{\frac{L}{6}} \leq 2  T^{\frac{1}{6}} \leq \sqrt{T}
	\]
	as soon as $T \geq 8$; 
	\item if $d = 2$, $\delta_\lvl^{-1/2} =  \sqrt{\nlvls} \leq \sqrt{\log_2 (T)/ 2} \leq \sqrt{T}$;
	\item if $d \geq 3$,
	\[
		\delta_\lvl^{-1/2} \leq \delta_1^{-1/2} =  2^{-\frac{d-2}{6}} \bigg(\sum_{\lvl=1}^L 2^{\lvl \big(\frac{d-2}{3}\big)}\bigg)^{1/2} \leq 2^{1+L\big(\frac{d-2}{6}\big)} \leq 2 T^{\frac{d-2}{6d}} \leq 2 T^{\frac{1}{6}} \leq \sqrt{T}
	\]
	as soon as $T \geq 8$. 
\end{itemize}

\end{proof}

\section{FTRL Analysis}\label{app:tech}

\subsection{Helper Lemmas}

\begin{lemma}\label{lem:ftrl-decomp}
    For $t \geq 1$, let $(\psi_t)_{t}$ be a sequence of functions where $\psi_t \colon \R^K \rightarrow (-\infty,+\infty]$. %
    Additionally, let $(z_t)_{t}$ be an arbitrary sequence of vectors in $\R^K$.
    Define $p_1 \coloneqq \argmin_{p \in \Delta_K} \psi_1(p)$ and for $t \geq 2$,
    \[
        p_t \coloneqq \argmin_{p \in \Delta_K} \ban{\summ_{s \leq t-1} z_s,p} + \psi_t(p) \,.
    \]
    Then, assuming that the sequence $(p_t)_{t}$ is well-defined, it holds for any comparator $q \in \Delta_K$ and horizon $T \geq 1$ that
    \begin{multline*}
        \sum_{t=1}^T \inprod{z_t, p_t - q} \leq \psi_{T+1}(q)-\psi_{1}(p_1) + \sum_{t=1}^T \psi_{t}(p_{t+1}) - \psi_{t+1}(p_{t+1})
        \\+ \sum_{t=1}^T F_{t}\brb{p_t,\summ_{s \leq t} z_s}-F_{t}\brb{p_{t+1},\summ_{s \leq t} z_s} \,,
    \end{multline*}
    where for $t \geq 1$, $z \in \R^d$, and $p \in \Delta_K$; we define $F_t(p,z) \coloneqq \inprod{z,p} + \psi_t(p)$.
\end{lemma}

\begin{proof}
    For any $t \in [T]$, $p_t = \argmin_{p \in \Delta_K} F_{t}\brb{p,\summ_{s \leq t-1} z_s}$.
    So,
    \begin{align*}
        \sum_{t=1}^T \inprod{z_t, p_t - q} &= \sum_{t=1}^T \inprod{z_t, p_t} + \psi_{T+1}(q) - F_{T+1}(q, \summ_{s \leq T} z_s) \\
        &\leq \sum_{t=1}^T \inprod{z_t, p_t} + \psi_{T+1}(q) - F_{T+1}(p_{T+1}, \summ_{s \leq T} z_s) \,.
    \end{align*}
    Moving on,
    \begin{align*}
        - F_{T+1}(p_{T+1}, \summ_{s \leq t} z_s) &= - \psi_{1}(p_1) +  \psi_{1}(p_1) - F_{T+1}(p_{T+1}, \summ_{s \leq T} z_s) \\
        &= - \psi_{1}(p_1) +  F_{1}(p_{1}, \zeros) - F_{T+1}(p_{T+1}, \summ_{s \leq T} z_s) \\
        &= - \psi_{1}(p_1) + \sum_{t=1}^T F_{t}(p_{t}, \summ_{s \leq t-1} z_s) - F_{t+1}(p_{t+1}, \summ_{s \leq t} z_s) \,.
    \end{align*}
    Lastly at any round $t$,
    \begin{align*}
        &\inprod{z_t, p_t} + F_{t}(p_{t}, \summ_{s \leq t-1} z_s) - F_{t+1}(p_{t+1}, \summ_{s \leq t} z_s) \\
        &\quad = F_{t}(p_{t}, \summ_{s \leq t} z_s) - F_{t+1}(p_{t+1}, \summ_{s \leq t} z_s) \\
        &\quad = F_{t}(p_{t}, \summ_{s \leq t} z_s) - F_{t}(p_{t+1}, \summ_{s \leq t} z_s) + \psi_{t}(p_{t+1}) - \psi_{t+1}(p_{t+1}) \,,
    \end{align*}
    concluding the proof.
\end{proof}

We say that a convex function $f \colon \R \rightarrow (-\infty,+\infty]$ is Legendre if it is proper, closed, essentially smooth (see \citep[Chapter 26]{rockafellar1970} for a definition) and strictly convex on $\interior(\dom(f))$. 

\begin{lemma} \label{lem:simple-variance-generic}
    Let $f \colon \R \rightarrow (-\infty,+\infty]$ be a convex function of the Legendre type satisfying $\R_{++} \subseteq \dom(f)$. For any $p \in \Delta_K$, define
    \begin{equation*}
        \psi(p) \coloneqq \sum_{j \in [L]} \mu_j \sum_{v \in V_j} f(p[v]) \,,
    \end{equation*}
    where $\mu_j > 0$ is a fixed weight for level $j$.
    Fix two vectors $y, z \in \R^K$, and let
    \[
        q \coloneqq \argmin_{p \in \Delta_K} \inprod{p,y} + \psi(p) \,.
    \]
    Then, for any $(c_j)_{j \in [L]} \in \R^L$ and $\ell \colon \cV \rightarrow \R$ such that $z(a) = \sum_{\ancof{v}{a}} \ell(v)$, it holds that
    \begin{multline*}
         \inprod{q,y+z} + \psi(q) - \min_{p \in \Delta_K} \brb{\inprod{p,y+z} + \psi(p)} \\
         \leq \sum_{j \in [L]} \mu_j \sum_{v \in V_j} D_{f^*}\brb{f'(p[v]) -  (\ell(v) - c_j) / \mu_j  \big\|f'(p[v])} \,.
    \end{multline*}
\end{lemma}
\begin{proof}
    Associate to every $v \in \cV$ a vector $h_v \in \R^K$ such that $h_v(a) \coloneqq \indicator{\ancof{v}{a}}$; hence, for any vector $x \in \R^K$,
    \begin{equation*}
        x^\top h_v = \sum_{a \in \cA \colon \ancof{v}{a} } x(a) = x[v] \,.
    \end{equation*}
    Moreover, let $\eta_v \coloneqq 1/\mu_{\level{v}}$.
    So, for any $p \in \Delta_K$,
    \begin{equation*}
        \psi(p) \coloneqq  \sum_{v \in \cV} \frac{1}{\eta_v} f(p^\top h_v) \,.
    \end{equation*}
    Hence, the result would follow from \Cref{prop:general-variance-generic} if its conditions on $f$ and $\{h_v\}_{v \in \cV}$ hold here.
    By the assumption that $f$ is Legendre, it is indeed proper, closed, and essentially smooth as required.
    Next, note that for any $v \in \cV$, the definition of $h_v$ together with the assumption that $\R_{++} \subseteq \dom(f)$ guarantee that $\R^K_{++} \subseteq \{ x \in \R^K \colon h_v^\top x \in \dom(f)\}$.
    Hence, 
    \[\relint(\{ x \in \R^K \colon h_v^\top x \in \dom(f)\}) = \interior(\{ x \in \R^K \colon h_v^\top x \in \dom(f)\}) \supseteq \R^K_{++} \,.\]
    And since $\relint(\Delta_K) \subset \R^K_{++}$ and $\relint(\Delta_K) \neq \emptyset$, it holds, as required, that
    \[
        \relint(\Delta_K) \cap \brb{ \cap_{v \in \cV} \relint(\{ x \in \R^K \colon h_v^\top x \in \dom(f)\}) } \neq \emptyset \,.
    \]
\end{proof}

\begin{lemma} \label{lem:variational-bound}
    Let $f \colon \R_+ \rightarrow (-\infty,+\infty]$ be a convex function of the Legendre type satisfying $\R_{++} \subseteq \dom(f)$. For any $p \in \Delta_K$, define
    \begin{equation*}
        \psi(p) \coloneqq \sum_{j \in [L]} \mu_j \sum_{v \in V_j} f(p[v]) \,,
    \end{equation*}
    where $\mu_j > 0$ is a fixed weight for level $j$.
    It holds for any $p \in \Delta_K$ and $\zeta \colon \cV \rightarrow \R$ that
    \begin{align*}
        -\psi(p) \leq \sum_{a \in \cA} p(a) \sum_{\ancof{v}{a}} \zeta(v) + \sum_{j \in [L]} \mu_j \sum_{v \in V_j} f^*(- \zeta(v)/\mu_j) \,.
    \end{align*}
\end{lemma}
\begin{proof}
    Using the same arguments in the proof of \Cref{lem:simple-variance-generic} (note that $\mu_j f$ is Legendre when $f$ is Legendre), we can leverage the first result of \Cref{lem:strong-duality-sum-one-d} (taking $u = \bm{0}$) to obtain that
    \begin{align*}
        \psi(p) + \sum_{a \in \cA} p(a) \sum_{\ancof{v}{a}} \zeta(v) \geq - \sum_{j \in [L]} \mu_j \sum_{v \in V_j} f^*(- \zeta(v)/\mu_j) \,,
    \end{align*}
    where we have used that $(\mu_j f)^*(\cdot) = \mu_j f^*(\cdot/\mu_j)$.
\end{proof}

\subsubsection{Specific Results for the nested Tsallis Regularizer}

\begin{lemma} \label{lem:bregman-bound-tsallis}
    Let $f(x) = \frac{2}{\eta} (x-\sqrt{x})$ with $\eta >0 $. Then, for any $p >0$, $c > 1$, and $y \in \R$, such that $y \geq -\frac{1}{c \eta \sqrt{p}}$, it holds that
    \[
        D_{f^*}\brb{f'(p) - y \big\|f'(p)} \leq \frac{c}{c-1} \eta p^{\nicefrac{3}{2}} y^2 \,.
    \]
\end{lemma}
\begin{proof}
    We have that $f'(x) = \frac{2}{\eta} (1-\frac{1}{2\sqrt{x}})$ and $f^*(z) = \frac{1}{\eta} \cdot \frac{1}{2-\eta z}$ for $\eta z < 2$.
    Note that
    \[
        \eta f'(p)  = 2 - \frac{1}{\sqrt{p}} < 2 \text{  and  } \eta f'(p) - \eta y  = 2 - \frac{1}{\sqrt{p}} - \eta y \leq 2 - \frac{1}{\sqrt{p}} + \frac{1}{c \sqrt{p}}  < 2 \,,
    \]
    where we have used that $p>0$ and that $\eta y \geq -\frac{1}{c\sqrt{p}}$ with $c > 1$.
    Hence,
    \begin{align*}
        D_{f^*}\brb{f'(p) - y \big\|f'(p)} &=
        yp + f^*(f'(p) - y) - f^*(f'(p)) \\
        &= yp + \frac{1}{\eta} \sqrt{p} \bbrb{\frac{1}{1+\eta y \sqrt{p}} - 1} \\
        &= yp - \frac{y p}{1+\eta y \sqrt{p}} \\
        &= \frac{\eta p^{\nicefrac{3}{2}} y^2}{1+\eta y \sqrt{p}} \\
        &\leq \frac{c}{c-1} \eta p^{\nicefrac{3}{2}} y^2 \,,
    \end{align*}
    where the last inequality uses the assumption that $\eta y \geq -\frac{1}{c\sqrt{p}}$.
\end{proof}

\begin{lemma} \label{lem:variational-bound-tsallis}
    Fix any $p \in \Delta_K$, $\mu_j > 0$ for $j \in [L]$, $\lambda \in \R$, $g \in \R^K$, and $\zeta \colon \cV \rightarrow \R$ such that $g(a) + \lambda \geq \sum_{\ancof{v}{a}} \zeta(v)$ for $a \in \cA$ and $\zeta(v)/\mu_j > - 2$ for $j \in [L]$ and $v \in V_j$.
    Then, it holds that
    \begin{align*}
        2 \sum_{j \in [L]} \mu_j \sum_{v \in V_j}  \sqrt{p_t[v]} - p_t[v]  \leq \inprod{g,p} + \lambda + \sum_{j \in [L]} \mu_j \sum_{v \in V_j} \frac{1}{2 + \zeta(v)/\mu_j} \,.
    \end{align*}
    In particular, for any $\tilde{\zeta} \colon \cV \rightarrow \R$ such that $(i)$ $\tilde{\zeta}(v) \geq 0$ $\forall v \in V$, $(ii)$ $g(a) \geq \sum_{\ancof{v}{a}} \tilde{\zeta}(v)$, and $(iii)$ $|\{v \in V_j \mid \tilde{\zeta}(v) = 0\}| \leq 1$ for all $j \in [L]$; it holds that
    \begin{align*}
        2 \sum_{j \in [L]} \mu_j \sum_{v \in V_j}  \sqrt{p_t[v]} - p_t[v]  \leq \inprod{g,p} + \sum_{j \in [L]} \mu^2_j \sum_{v \in V_j \colon \tilde{\zeta}(v) \neq 0} \frac{1}{ \tilde{\zeta}(v)} \,.
    \end{align*}
\end{lemma}
\begin{proof}
    Let $f(x) = 2 (x - \sqrt{x})$. Then $f^*(y) = 1/(2-y)$ for $y <2 $ and $+\infty$ elsewhere.
    Hence, applying \Cref{lem:variational-bound} with this function gives that
    \begin{align*}
        2 \sum_{j \in [L]} \mu_j \sum_{v \in V_j}  \sqrt{p_t[v]} - p_t[v]  &\leq \sum_{a \in \cA} p(a) \sum_{\ancof{v}{a}} \zeta(v) + \sum_{j \in [L]} \mu_j \sum_{v \in V_j} \frac{1}{2 + \zeta(v)/\mu_j} \\
        &\leq \inprod{g,p} + \lambda + \sum_{j \in [L]} \mu_j \sum_{v \in V_j} \frac{1}{2 + \zeta(v)/\mu_j} \,,
    \end{align*}
    where the condition $\zeta(v)/\mu_j > -2$ for $v \in V_j$ ensures we respect the effective domain of $f^*$.
    Now, for the second part, choose $\lambda = -\sum_{j \in [L]} \mu_j$ and $\zeta(v) = \tilde{\zeta}(v) - \mu_j$ for $v \in V_j$, getting as a consequence that
    \begin{align*}
        2 \sum_{j \in [L]} \mu_j \sum_{v \in V_j}  \sqrt{p_t[v]} - p_t[v]  
        &\leq \inprod{g,p} + \sum_{j \in [L]} \mu_j \lrb{\sum_{v \in V_j} \frac{1}{1 + \tilde{\zeta}(v)/\mu_j} -1 } \,.
    \end{align*}
    Finally, to obtain the sought result we use that ${1}/(1 + \tilde{\zeta}(v)/\mu_j) \geq {\mu_j}/ \tilde{\zeta}(v)$ for any $v$ such that $\tilde{\zeta}(v) \neq 0$ and that ${1}/(1 + \tilde{\zeta}(v)/\mu_j) = 1$ otherwise, which we assumed to be the case for at most one node per level.  
\end{proof}

\begin{lemma} \label{lem:tsallis-variance-simplification}
    Fix $p \in \Delta_d$ and $x \in \R^d$.
    It holds that for any $j \in [d]$ that
    \[
        \sum_{i \in [d]} p(i)^{\nicefrac{3}{2}} \brb{ \indicator{i=j} x(i) / p(i) - x(j) }^2 = x(j)^2 \lrb{ p(j)^{-\nicefrac{1}{2}} -2 p(j)^{\nicefrac{1}{2}} + \summ_{i \in [d]} p(i)^{\nicefrac{3}{2}}} \,.
    \]
    It also holds that
    \[
        \sum_{j\in[d]} p(j) \lrb{ p(j)^{-\nicefrac{1}{2}} -2 p(j)^{\nicefrac{1}{2}} + \summ_{i \in [d]} p(i)^{\nicefrac{3}{2}}} \leq 2 \sum_{j\in[d]} p(j)^{\nicefrac{1}{2}}  - p(j) \,.
    \]
\end{lemma}

\begin{proof}
    We have that
    \begin{align*}
        &\sum_{i \in [d]} p(i)^{\nicefrac{3}{2}} \brb{ \indicator{i=j} x(i) / p(i) - x(j) }^2 \\
        &\quad= \sum_{i \in [d]} p(i)^{\nicefrac{3}{2}} \brb{ \indicator{i=j} x(j) / p(j) - x(j) }^2 \\
        &\quad= \sum_{i \in [d]} p(i)^{\nicefrac{3}{2}} \brb{ \indicator{i=j} x(j)^2 / p(j)^2 + x(j)^2 - 2 \indicator{i=j} x(j)^2 / p(j) } \\
        &\quad= x(j)^2 \lrb{ p(j)^{-\nicefrac{1}{2}} -2 p(j)^{\nicefrac{1}{2}} + \summ_{i \in [d]} p(i)^{\nicefrac{3}{2}}} \,.
    \end{align*}
    Next,
    \begin{align*}
        &\sum_{j\in[d]} p(j) \lrb{ p(j)^{-\nicefrac{1}{2}} -2 p(j)^{\nicefrac{1}{2}} + \summ_{i \in [d]} p(i)^{\nicefrac{3}{2}}} \\
        &\quad= \sum_{j\in[d]} p(j)^{\nicefrac{1}{2}}  -2 \sum_{j\in[d]} p(j)^{\nicefrac{3}{2}} + \sum_{j\in[d]} p(j)^{\nicefrac{3}{2}} \\
        &\quad= \sum_{j\in[d]} p(j)^{\nicefrac{1}{2}}  - \sum_{j\in[d]} p(j)^{\nicefrac{3}{2}} \\  
        &\quad= \sum_{j\in[d]} p(j)^{\nicefrac{1}{2}}  (1 - p(j) ) \\
        &\quad = \sum_{j\in[d]} p(j)^{\nicefrac{1}{2}}  (1 - p(j)^{\nicefrac{1}{2}} )(1 + p(j)^{\nicefrac{1}{2}} ) \\
        &\quad\leq 2 \sum_{j\in[d]} p(j)^{\nicefrac{1}{2}}  (1 - p(j)^{\nicefrac{1}{2}} ) =  2 \sum_{j\in[d]} p(j)^{\nicefrac{1}{2}}  - p(j) \,.
    \end{align*}
\end{proof}

\subsection{Generic Regret Guarantees for the Nested Tsallis Regularizer}\label{app:analysis}

\begin{lemma} \label{lem:ftrl-nested-tsallis-generic}
    Let $(\eta_{t,j})_{t \geq 1, j \in [L]}$ be a sequence of positive numbers.
    Accordingly, for every $t \geq 1$, define $\psi_t \colon \R^k \rightarrow (-\infty,+\infty]$ as
    \[
        \psi_t(x) \coloneqq 2 \sum_{j \in [L]} \frac{1}{\eta_{t,j}} \sum_{v \in V_j}  \brb{x[v] - \sqrt{x[v]}}
    \]
    for $x \in \R^K$.
    Let $(z_t)_t$ be a sequence of loss vectors in $\R^K$, and define \[
        p_1 \coloneqq \argmin_{p \in \Delta_K} \psi_1(p) \quad \text{and for}\:\: t \geq 2, \:\:
        p_t \coloneqq \argmin_{p \in \Delta_K} \ban{\summ_{s \leq t-1} z_s,p} + \psi_t(p) \,.
    \]
    Fix $c>1$, and for each $t \geq 1$, fix a sequence $(c_{t,j})_{j \in [L]} \in \R^L$ and a function $\ell_t \colon \cV \rightarrow \R$ such that
    \[
        z_t(a) =  \sum_{\ancof{v}{a}} \ell_t(v) \:\:\forall a \in \cA \quad \text{and} \quad \ell_t(v) - c_{t,j} \geq - \frac{1}{c \eta_{t,j} \sqrt{p_t[v]}} \:\:\forall j  \in [L] \,,\, v \in V_j \,.
    \]
    Then, it holds for horizon $T \geq 1$ and comparator $q \in \Delta_k$ that
    \begin{multline*}
        \sum_{t=1}^T \inprod{z_t, p_t - q} \leq 2 \sum_{t=1}^{T+1} \sum_{j \in [L]} \tilde{\eta}_{t,j}  \sum_{v \in V_j}  \brb{\sqrt{p_t[v]} - p_t[v]}  
        \\ + \frac{c}{c-1} \sum_{t=1}^T  \sum_{j \in [L]} \eta_{t,j} \sum_{v \in V_j}  \brb{p_t[v]}^{\nicefrac{3}{2}} (\ell_t(v) - c_{t,j})^2 \,,
    \end{multline*}
    where $\tilde{\eta}_{1,j} = \eta_{1,j}^{-1}$ and for $t\geq 2$, $\tilde{\eta}_{t,j} = \eta_{t,j}^{-1} - \eta_{t-1,j}^{-1}$.
\end{lemma}
\begin{proof}
    This is obtained by combining \Cref{lem:ftrl-decomp,lem:simple-variance-generic,lem:bregman-bound-tsallis}, and using the fact that $\psi_{T+1}(q) \leq 0$.
\end{proof}

\begin{lemma} \label{lem:ftrl-nested-tsallis-sparse}
    Fix a sequence of positive numbers $(R_j)_{j \in [L]}$ and a sequence $(\xi_j)_{j \in [L]}$ of weights in $[0,1]$. 
    Moreover, with $t \geq 1$, let $(y_t)_t$ be a random sequence of loss functions each mapping $\cA$ to $\R$,$(A_t)_t$ a random sequence of actions from $\cA$, $(M_t)_t$ a random sequence of subsets of $[L]$, and $(\ell_t)_t$ a random sequence of functions mapping $\cV$ to $\R$ such that $\ell_t(v) \in [0, R_j]$ for all $j\in[L]$, $v \in V_j$, and $t \geq 1$.
    Accordingly, define
    \[
        p_1 \coloneqq \argmin_{p \in \Delta_K} \psi_1(p) \quad \text{and for}\:\: t \geq 2, \:\:
        p_t \coloneqq \argmin_{p \in \Delta_K} \ban{\summ_{s \leq t-1} z_s,p} + \psi_t(p) \,,
    \]
    where for $t \geq 1$ and $a \in \cA$, $\displaystyle z_t(a) \coloneqq \sum_{\ancof{v}{a}} \indicator{ \ancof{v}{A_t}, \level{v} \in M_t } \frac{{\ell}_t(v)}{\xi_{\level{v}} p_t[v]}$; and for $x \in \R^d$,
    \[
        \psi_t(x) \coloneqq \frac{2 \sqrt{c}}{\sqrt{(c-1)}} \sum_{j \in [L]} \frac{R_j}{\sqrt{\xi_j}} \max\Bcb{ \sqrt{t}, \sqrt{{c(c-1)}/{\xi_j} } } \sum_{v \in V_j} x[v] - \sqrt{x[v]} 
    \]
    with some constant $c \geq 2$. Assume there exists a filtration $(\cF_t)_{t \geq 0}$ such that 
    \begin{enumerate}[nosep]
        \item $(p_t)_{t \geq 1}$ is predictable \wrt   $(\cF_t)_{t \geq 0}$,
        \item $\pr_t(A_t = a) = p_t(a)\:\: \forall a\in \cA$, 
        \item $\pr_t(j \in M_t) = \xi_j\:\:\forall j \in [L]$,
        \item $\indicator{\ancof{v}{A_t}} \ind \indicator{\level{v} \in M_t} \mid \cF_{t-1}\:\: \forall v \in \cV$,
        \item $\sum_{\ancof{v}{a}} \E_t[\ell_t(v) \mid \ancof{v}{A_t}, \level{v} \in M_t] - \sum_{\ancof{w}{a'}} \E_t[\ell_t(w) \mid \ancof{w}{A_t}, \level{w} \in M_t] = \E_t[y_t(a) - y_t(a')] \:\: \forall a,a' \in \cA$,
    \end{enumerate}
    where $\pr_t(\cdot) = \pr(\cdot \mid \cF_{t-1})$ and $\E_t[\cdot] = \E[\cdot \mid \cF_{t-1}]$.
    Then, it holds for any horizon $T \geq 1$ that
    \begin{multline*}
        \E \lsb{\sum_{t=1}^T y_t (A_t)} - \min_{a \in \cA} \E \lsb{ \sum_{t=1}^T y_t (a)} 
        \leq \frac{4 \sqrt{c}}{\sqrt{(c-1)}} \sum_{t=1}^{T} \frac{1}{\sqrt{t}} \sum_{j \in [L]} \frac{R_j}{\sqrt{\xi_j}} \E \sum_{v \in V_j}  \brb{\sqrt{p_t[v]} - p_t[v]} \\
        + 3c \sum_{j \in [L]} \frac{R_j}{\xi_j} \sqrt{|V_j|} \,.
    \end{multline*}
\end{lemma}

\begin{proof}
    Let $a^* \in \argmin_{a \in \cA} \E \lsb{ \sum_{t=1}^T y_t (a)}$ and let $q \in \Delta_k$ be entirely concentrated on $a^*$.
    For brevity, we will use $\tilde{\ell}_t(v)$ to denote $\E_t[\ell_t(v) \mid \ancof{v}{A_t}, \level{v} \in M_t]$ for $v \in \cV$.
    With this definition, it holds that 
    \begin{align*}
        &\E_t \bsb{ \indicator{ \ancof{v}{A_t}, \level{v} \in M_t } {{\ell}_t(v)} } \\
        &\qquad= \pr_t(\ancof{v}{A_t} , \level{v} \in M_t) \E_t\bsb{\ell_t(v) \indicator{ \ancof{v}{A_t}, \level{v} \in M_t } \mid \ancof{v}{A_t}  , \level{v} \in M_t } \\
        &\qquad\qquad+ \pr_t(\nancof{v}{A_t} \lor \level{v} \notin M_t) \E_t\bsb{\ell_t(v) \indicator{ \ancof{v}{A_t}, \level{v} \in M_t } \mid \nancof{v}{A_t}  \lor \level{v} \notin M_t} \\
        &\qquad= \E_t \bsb{ \indicator{ \ancof{v}{A_t}, \level{v} \in M_t } } \tilde{\ell}_t(v) = \E_t \bsb{ \indicator{ \ancof{v}{A_t}}} \E_t\bsb{ \indicator{\level{v} \in M_t }}  \tilde{\ell}_t(v) \,,
    \end{align*}
    where the last step holds via the assumed independence of $\indicator{ \ancof{v}{A_t}}$ and $\indicator{\level{v} \in M_t }$ conditioned on $\cF_{t-1}$.
    Note then that
    \begin{align*}
        \E \inprod{z_t, p_t} &= \E \sum_{a \in \cA} p_t(a) \sum_{\ancof{v}{a}} \indicator{ \ancof{v}{A_t}, \level{v} \in M_t } \frac{{\ell}_t(v)}{\xi_{\level{v}} p_t[v]}\\
        &= \E \sum_{a \in \cA} p_t(a) \sum_{\ancof{v}{a}} \E_t\bsb{\indicator{ \ancof{v}{A_t}, \level{v} \in M_t } {\ell}_t(v)} \frac{1}{\xi_{\level{v}} p_t[v]}\\
        &= \E \sum_{a \in \cA} p_t(a) \sum_{\ancof{v}{a}} \E_t\lsb{\indicator{ \ancof{v}{A_t}}} \E_t\lsb{\indicator{\level{v} \in M_t }} \frac{\tilde{\ell}_t(v)}{\xi_{\level{v}} p_t[v]}\\
        &= \E \sum_{a \in \cA} p_t(a) \sum_{\ancof{v}{a}} \tilde{\ell}_t(v)  \\
        &= \E \sum_{v \in \cV} \tilde{\ell}_t(v) \sum_{\descof{a}{v}} p_t(a)    \\
        &= \E \sum_{v \in \cV} \tilde{\ell}_t(v)  \E_t[\indicator{ \ancof{v}{A_t} }]   \\
        &= \E \sum_{v \in \cV} \tilde{\ell}_t(v)  \indicator{ \ancof{v}{A_t} } = \E \sum_{\ancof{v}{A_t}} \tilde{\ell}_t(v)  \,,
    \end{align*}
    where the first equality uses the definition of $z_t$; 
    the second uses the tower rule and the fact that $p_t$ is deterministic given $\cF_{t-1}$; 
    the fourth uses the definition of $p_t[v]$, the assumption that the law of $A_t$ conditioned on $\cF_{t-1}$ is $p_t$, and the assumption that $\pr_t(\{\level{v} \in M_t\} ) = \xi_{\level{v}}$;
    the sixth employs again the fact that the law of $A_t$ conditioned on $\cF_{t-1}$ is $p_t$; and the seventh again uses the tower rule and the fact that $\tilde{\ell}_t(v)$ is measurable \wrt $\cF_{t-1}$.
    In a similar fashion, it holds that
    \begin{align*}
        \E \inprod{z_t, {q}} &= \E \sum_{a \in \cA} {q}(a) \sum_{\ancof{v}{a}} \indicator{ \ancof{v}{A_t}, \level{v} \in M_t } \frac{{\ell}_t(v)}{\xi_{\level{v}} p_t[v]} \\
        &= \E \sum_{a \in \cA} {q}(a) \sum_{\ancof{v}{a}} \E_t\bsb{\indicator{ \ancof{v}{A_t}, \level{v} \in M_t } {\ell}_t(v)} \frac{1}{\xi_{\level{v}} p_t[v]} \\
        &= \E \sum_{a \in \cA} {q}(a) \sum_{\ancof{v}{a}} \E_t\lsb{\indicator{ \ancof{v}{A_t}}} \E_t\lsb{\indicator{\level{v} \in M_t }} \frac{\tilde{\ell}_t(v)}{\xi_{\level{v}} p_t[v]}\\
        &= \E \sum_{a \in \cA} {q}(a) \sum_{\ancof{v}{a}} \tilde{\ell}_t(v)  \,.
    \end{align*}
    Since, by assumption, $  \sum_{\ancof{v}{a}} \tilde{\ell}_t(v) - \sum_{\ancof{w}{a'}} \tilde{\ell}_t(w) = \E_t[y_t(a) - y_t(a')]$ for all $a,a' \in \cA$, combining the last two result with an application of the tower rule gets us that
    \begin{align*}
        \E \inprod{z_t, p_t - {q}} = \E\bsb{y_t(A_t) - \inprod{y_t, {q}} }\,.
    \end{align*}
    Overall, we have shown that
    \begin{align*}
        \E \lsb{\sum_{t=1}^T y_t (A_t)} - \min_{a \in \cA} \E \lsb{ \sum_{t=1}^T y_t (a)} 
        = \E \sum_{t=1}^T \inprod{z_t, p_t - {q}} \,.
    \end{align*}
    We can now invoke \Cref{lem:ftrl-nested-tsallis-generic} with some $c \geq 2$, $\eta_{t,j} = \frac{\sqrt{\xi_j}}{R_j} \min\bbcb{ \frac{\sqrt{c-1}}{\sqrt{c t}}, \frac{\sqrt{\xi_j}}{c} }$, and $c_{t,j}= \indicator{ j \in M_t }\frac{\ell_t(\anc{A_t}{j})}{\xi_j}$ to get that
    \begin{multline*}
        \hspace{-1em}\sum_{t=1}^T \inprod{z_t, p_t - {q}} \leq 2c \sum_{j \in [L]} \frac{R_j}{\xi_j} \sum_{v \in V_j}  \brb{\sqrt{p_1[v]} - p_1[v]}+ \frac{2 \sqrt{c}}{\sqrt{(c-1)}} \sum_{t=2}^{T+1} \frac{1}{\sqrt{t}} \sum_{j \in [L]} \frac{R_j}{\sqrt{\xi_j}} \sum_{v \in V_j}  \brb{\sqrt{p_t[v]} - p_t[v]}  
        \\ + \frac{\sqrt{c}}{\sqrt{c-1}} \sum_{t=1}^T \frac{1}{\sqrt{t}} \sum_{j \in [L]} \frac{\sqrt{\xi_j}}{R_j} \sum_{v \in V_j}  \brb{p_t[v]}^{\nicefrac{3}{2}} \lrb{\indicator{ v=\anc{A_t}{j}, j \in M_t } \frac{{\ell}_t(v)}{\xi_{j} p_t[v]}  - \indicator{ j \in M_t }\frac{\ell_t(\anc{A_t}{j})}{\xi_j}}^2 \,,
    \end{multline*}    
    where we have used that $\tilde{\eta}_{1,j} = \eta^{-1}_{1,j} = \frac{c R_j}{\xi_j}$, and for $t \geq 2$, 
    \begin{align*}
        \tilde{\eta}_{t,j} &= \frac{ R_j}{\sqrt{\xi_j}} \max\bbcb{ \frac{\sqrt{c t}}{\sqrt{c-1}}, \frac{c}{\sqrt{\xi_j} } } - \frac{ R_j}{\sqrt{\xi_j}} \max\bbcb{ \frac{\sqrt{c (t-1)}}{\sqrt{c-1}}, \frac{c}{\sqrt{\xi_j} } } \\
        &\leq  \frac{\sqrt{c} R_j}{\sqrt{(c-1)\xi_j}} (\sqrt{t}-\sqrt{t-1}) \\
        &\leq \frac{\sqrt{c} R_j}{\sqrt{(c-1)\xi_j t}} \,.
    \end{align*}
    Note also that for every $t \geq 1$, $j \in [L]$, and $v \in V_j$, 
    \begin{align*}
        & \indicator{ v=\anc{A_t}{j}, j \in M_t } \frac{{\ell}_t(v)}{\xi_{j} p_t[v]}  - \indicator{ j \in M_t }\frac{\ell_t(\anc{A_t}{j})}{\xi_j} \\
        &\qquad=\frac{\indicator{ j \in M_t }}{\xi_j } \lrb{\indicator{ v=\anc{A_t}{j}} \frac{{\ell}_t(v)}{ p_t[v]}  - \ell_t(\anc{A_t}{j})} \\
        &\qquad\geq - \frac{\ell_t(\anc{A_t}{j})}{\xi_j}  \\
        &\qquad\geq - \frac{R_j}{\xi_j }  \geq -  \frac{1}{c \eta_{t,j}} \geq -  \frac{1}{c \eta_{t,j} \sqrt{p_t[v]}} \,;
    \end{align*}
    thus, the conditions of \Cref{lem:ftrl-nested-tsallis-generic} are indeed satisfied.
    
    Moving on, it holds that 
    \begin{align*}
        &\E_t \sum_{v \in V_j}  \brb{p_t[v]}^{\nicefrac{3}{2}} \lrb{\indicator{ v=\anc{A_t}{j}, j \in M_t } \frac{{\ell}_t(v)}{\xi_{j} p_t[v]}  - \indicator{ j \in M_t }\frac{\ell_t(\anc{A_t}{j})}{\xi_j}}^2 \\
        &\qquad= \E_t \frac{\indicator{ j \in M_t }}{\xi^2_j} \sum_{v \in V_j}  \brb{p_t[v]}^{\nicefrac{3}{2}} \lrb{\indicator{ v=\anc{A_t}{j}} \frac{{\ell}_t(v)}{ p_t[v]}  - {\ell_t(\anc{A_t}{j})}}^2 \\
        &\qquad= \E_t \frac{\indicator{ j \in M_t } {\ell}^2_t(\anc{A_t}{j}) }{\xi^2_j} \lrb{ (p_t[\anc{A_t}{j}])^{-\nicefrac{1}{2}} -2 (p_t[\anc{A_t}{j}])^{\nicefrac{1}{2}} + \summ_{v \in V_j} (p_t[v])^{\nicefrac{3}{2}}} \\
        &\qquad \leq R_j^2 \E_t \frac{\indicator{ j \in M_t }  }{\xi^2_j} \lrb{ (p_t[\anc{A_t}{j}])^{-\nicefrac{1}{2}} -2 (p_t[\anc{A_t}{j}])^{\nicefrac{1}{2}} + \summ_{v \in V_j} (p_t[v])^{\nicefrac{3}{2}}} \\
        &\qquad = R_j^2 \E_t \sum_{w \in V_j} \frac{\indicator{ j \in M_t } \indicator{ \ancof{w}{A_t} }  }{\xi^2_j} \lrb{ (p_t[w])^{-\nicefrac{1}{2}} -2 (p_t[w])^{\nicefrac{1}{2}} + \summ_{v \in V_j} (p_t[v])^{\nicefrac{3}{2}}} \\
        &\qquad = R_j^2  \sum_{w \in V_j} \frac{\E_t[\indicator{ j \in M_t }] \E_t[\indicator{ \ancof{w}{A_t} }]  }{\xi^2_j} \lrb{ (p_t[w])^{-\nicefrac{1}{2}} -2 (p_t[w])^{\nicefrac{1}{2}} + \summ_{v \in V_j} (p_t[v])^{\nicefrac{3}{2}}} \\
        &\qquad = \frac{R_j^2}{\xi_j}  \sum_{w \in V_j} p_t[w] \lrb{ (p_t[w])^{-\nicefrac{1}{2}} -2 (p_t[w])^{\nicefrac{1}{2}} + \summ_{v \in V_j} (p_t[v])^{\nicefrac{3}{2}}} \\
        &\qquad\leq 2 \frac{R_j^2}{\xi_j} \sum_{v \in V_j} \sqrt{p_t[v]} - p_t[v]  \,,
    \end{align*}
    where the second equality is an application of the first part of \Cref{lem:tsallis-variance-simplification}, the fourth equality uses the conditional independence of $\indicator{\ancof{w}{A_t}}$ and $\indicator{\level{w} \in M_t }$ given $\cF_{t-1}$, 
    and the last inequality is an application of the second part of \Cref{lem:tsallis-variance-simplification}.
    Putting everything together, obtain that
    \begin{align*}
        &\E \lsb{\sum_{t=1}^T y_t (A_t)} - \min_{a \in \cA} \E \lsb{ \sum_{t=1}^T y_t (a)} \\
        &\quad \leq \frac{4 \sqrt{c}}{\sqrt{(c-1)}} \sum_{t=1}^{T} \frac{1}{\sqrt{t}} \sum_{j \in [L]} \frac{R_j}{\sqrt{\xi_j}} \E \sum_{v \in V_j}  \brb{\sqrt{p_t[v]} - p_t[v]}
        +  2c \sum_{j \in [L]} \frac{R_j}{\xi_j} \E \sum_{v \in V_j}  \brb{\sqrt{p_1[v]} - p_1[v]} \\ 
        &\hspace{5em}+ \frac{2 \sqrt{c}}{\sqrt{(c-1)}} \cdot \frac{1}{\sqrt{T+1}} \sum_{j \in [L]} \frac{R_j}{\sqrt{\xi_j}} \E \sum_{v \in V_j}  \brb{\sqrt{p_{T+1}[v]} - p_{T+1}[v]} \,.
    \end{align*} 
    The theorem then follows since the sum of the last two terms is no larger than $3c \sum_{j \in [L]} \frac{R_j}{\xi_j} \sqrt{|V_j|}$.
\end{proof}

\section{Foundational Results}\label{app:foundation}

\begin{lemma} \label{lem:strong-duality}
    Fix a vector $z \in \R^d$ and let $\psi \colon \R^d \rightarrow (-\infty, +\infty]$ be a proper and closed convex function. 
    Assume further that $\relint(\Delta_d) \cap \relint(\dom(\psi)) \neq \emptyset$.
    Then, it holds that 
    \begin{equation*}
        \min_{p \in \Delta_d} z^\top p + \psi(p) = \max_{u \in \R^d} - \psi^*(u - z) + \minco{u}{i} \,,
    \end{equation*}
    and that
    \begin{equation*}
        u^* - z \in \partial \psi(p^*) \,,
    \end{equation*}
    for $p^* \in \argmin_{p \in \Delta_d} z^\top p + \psi(p)$ and $u^* \in \argmax_{u \in \R^d} - \psi^*(u - z) + \minco{u}{i}$\,.
    If, additionally, $\psi$ is essentially smooth, then $u^*$ is unique, $p^* \in \Delta_d \cap \interior(\dom(\psi))$, and 
    \begin{equation*}
        u^* - z = \nabla \psi(p^*) \,.
    \end{equation*}
    If $\psi$ is also strictly convex on $\interior(\dom(\psi))$, then $p^*$ too is unique.
\end{lemma}
\begin{proof}
    Define $\Phi \colon \R^d \rightarrow (-\infty, + \infty]$ as
    \begin{equation*}
        \Phi(x) \coloneqq z^\top x + \psi(x)  
    \end{equation*}
    for all $x \in \R^d$. 
    This is a proper and closed convex function satisfying that $\dom(\Phi) = \dom(\psi)$ and $\Phi^*(y) = \psi^*(y - z)$ for any $y \in \R^d$.
    Define also $g \colon \R^d \rightarrow [-\infty, + \infty)$ as 
    \[
        g(x) \coloneqq 
        \begin{cases}
            0 \: &\text{if}\: x \in \Delta_d \\
            -\infty \: &\text{otherwise}\,,
        \end{cases}
    \]
    which is a closed and proper \emph{concave} function satisfying that $\dom(g)$ = $\Delta_d$, $g^*(y) = \minco{y}{i}$ for any $y \in \R^d$, and $\dom(g^*) = \R^d$.
    Note then that $ \relint(\dom(g)) \cap \relint(\dom(\Phi)) \neq \emptyset$ since $\relint(\Delta_d) \cap \relint(\dom(\psi)) \neq \emptyset$ by assumption.
    Moreover, $ \relint(\dom(g^*)) \cap \relint(\dom(\Phi^*)) \neq \emptyset$ since $\dom(g^*) = \R^d$ and $\dom(\Phi^*)$ is not empty (because $\Phi^*$ is proper).
    Hence, Theorem 31.1 in \citep{rockafellar1970} gives that
    \[
        \inf_{x \in \R^d} \Phi(x) - g(x) = \sup_{y \in \R^d} g^*(y) - \Phi^*(y) 
    \]
    with both the infimum and the supremum being attained, which entails the first statement of the lemma.
    Now, fixing $p^* \in \argmin_{p \in \Delta_d} \Phi(p)$ and $u^* \in \argmax_{u \in \R^d} - \psi^*(u - z) + \minco{u}{i}$, it holds that
    \begin{align*}
        \Phi(p^*) &= - \psi^*(u^* - z) + \minco{u^*}{i} \\
        &= \inf_{x \in \R^d} (z - u^*)^\top x + \psi(x) + \minco{u^*}{i} \\
        &= \minco{u^*}{i} + \inf_{x \in \R^d} \Phi(x)  - x^\top u^* \\
        &\leq \minco{u^*}{i} +  \Phi(p^*)  - (p^*)^\top u^* \leq \Phi(p^*) \,,
    \end{align*}
    where the last inequality just uses that $p^* \in \Delta_d$.
    We can then conclude that $p^* \in \argmin_{x \in \R^d} \psi(x)  + x^\top (z - u^*)$, or equivalently
    \begin{align*}
        p^* \in \argmax_{x \in \R^d} x^\top (u^* - z) - \psi(x) \,,
    \end{align*}
    implying via Theorem 23.5 in \citep{rockafellar1970} that 
    \begin{equation} \label{eq:p-star-subgradient}
        u^* - z \in \partial \psi(p^*) \,,
    \end{equation}
    which gives the second statement.

    Moving on, if $\psi$ is essentially smooth, Theorem 26.1 in \citep{rockafellar1970} gives that $\partial \psi(x) = \{\nabla \psi(x)\}$ for $x \in \interior(\dom(\psi))$ and $\partial \psi(x) = \emptyset$ otherwise. It is then implied by \eqref{eq:p-star-subgradient} that $u^*$ is unique (since $\partial \psi(p^*)$ is non-empty and must be a singleton), that $p^* \in \interior(\dom(\psi))$, and that $u^* - z = \nabla \psi(p^*)$.
    If $\psi$ is additionally strictly convex on $\interior(\dom(\psi))$, Corollary 26.3.1 in \citep{rockafellar1970} gives that $\partial \psi$ is a one-to-one mapping, hence the uniqueness of $u^*$ implies via \eqref{eq:p-star-subgradient} the uniqueness of $p^*$.
\end{proof}

\begin{lemma} \label{lem:essentially-smooth-sum}
    For every $i$ in a finite index set $\cI$, let $\psi_i \colon \R^d \rightarrow (-\infty, +\infty]$ be a proper and closed convex function, and assume that $\cap_{i \in \cI} \relint(\dom(\psi_i)) \neq \emptyset$.
    If every $\psi_i$ is essentially smooth, then so is $\sum_{i \in \cI} \psi_i$.
\end{lemma}
\begin{proof}
    Let $\psi \coloneqq \sum_{i \in \cI} \psi_i$, which is also a proper and closed convex function.
    Theorem 26.1 in \citep{rockafellar1970} states that a proper and closed convex function $f \colon \R^d \rightarrow (\infty, +\infty]$ is essentially smooth if and only if $\partial f(x)$ contains at most one element for any $x \in \R^d$, which hence holds for every $\psi_i$.
    Since $\cap_{i \in \cI} \relint(\dom(\psi_i)) \neq \emptyset$, \citep[Proposition 5.4.6]{bertsekas2009} gives that 
    $\partial \psi(x) = \sum_{i \in \cI} \partial \psi_i(x)$
    for all $x \in \R^d$.
    Thus, $\partial \psi(x)$ too contains at most one element, implying the essential smoothness of $\psi(x)$.
\end{proof}

\begin{lemma} \label{lem:strong-duality-sum}
    Fix a vector $z \in \R^d$ and a finite index set $\cI$.
    For every $i \in \cI$, let $\psi_i \colon \R^d \rightarrow (-\infty, +\infty]$ be a proper and closed convex function. 
    Assume further that $\relint(\Delta_d) \cap \brb{\cap_{i \in \cI} \relint(\dom(\psi_i))} \neq \emptyset$. %
    Then, it holds that 
    \begin{equation*}
        \min_{p \in \Delta_d} z^\top p + \sum_{i \in \cI} \psi_i(p) = \max_{u \in \R^d} \sup_{\sum_{i \in \cI} y_i = u -z}  - \sum_{i \in \cI} \psi^*_i(y_i) + \minco{u}{i}\,,
    \end{equation*}
    and for any $p^* \in \Delta_d$ minimizing the L.H.S. and $u^* \in \R^d$ maximizing the R.H.S.\,,
    \begin{equation*}
        \min_{p \in \Delta_d} z^\top p + \sum_{i \in \cI} \psi_i(p) =  - \sum_{i \in \cI} \psi^*_i(g_i) + \minco{u^*}{i} \,,
    \end{equation*}
    for some $(g_i)_{i \in \cI}$ satisfying $g_i \in \partial \psi_i(p^*)$ and $\sum_{i \in \cI} g_i = u^* - z$.
    If, additionally, every $\psi_i$ is essentially smooth, 
    then $u^*$ is unique, $p^* \in \Delta_d \cap \interior(\cap_{i \in \cI}\dom(\psi_i))$, $\sum_{i \in \cI} \nabla \psi_i(p^*) = u^* -z $, and
    \begin{equation*}
        \min_{p \in \Delta_d} z^\top p + \sum_{i \in \cI} \psi_i(p) =  - \sum_{i \in \cI} \psi^*_i(\nabla \psi_i(p^*)) + \minco{u^*}{i} \,.
    \end{equation*}
    If, moreover, $\sum_{i \in \cI} \psi_i$ is strictly convex on $\interior(\cap_{i \in \cI} \dom(\psi_i))$, then $p^*$ too is unique.
\end{lemma}

\begin{proof}
    For brevity, define
    \begin{equation*}
        \psi \coloneqq \sum_{i \in \cI} \psi_i \,.
    \end{equation*}
    This is a convex function that is closed \citep[Proposition 1.1.5]{bertsekas2009} and satisfies $\dom(\psi) = \cap_{i \in \cI}\dom(\psi_i)$.
    Hence, our assumption that $\relint(\Delta_d) \cap \brb{\cap_{i \in \cI} \relint(\dom(\psi_i))} \neq \emptyset$ implies that $\psi$ is proper seeing as $\dom(\psi) \neq \emptyset$.
    Moreover, Theorem 6.5 in \citep{rockafellar1970} gives that $\relint(\cap_{i \in \cI}\dom(\psi_i)) = \cap_{i \in \cI} \relint(\dom(\psi_i))$ seeing as $\cap_{i \in \cI} \relint(\dom(\psi_i)) \neq \emptyset$. 
    Hence, it holds that $\relint(\Delta_d) \cap \relint(\dom(\psi)) \neq \emptyset$.
    \Cref{lem:strong-duality} then gives that
    \begin{equation} \label{eq:strong-duality-closed-sum}
        \min_{p \in \Delta_d} z^\top p + \psi(p) = \max_{u \in \R^d} - \psi^*(u - z) + \minco{u}{i}  \,.
    \end{equation}
    Since $\cap_{i \in \cI} \relint(\dom(\psi_i)) \neq \emptyset$, Theorem 16.4 in \citep{rockafellar1970} gives that
    \begin{align*}
        \psi^*(u - z) = \inf_{\sum_{i \in \cI} y_i = u - z} \sum_{i \in \cI} \psi^*_i(y_i) \,,
    \end{align*}
    implying, as sought, that
    \begin{equation*}
        \min_{p \in \Delta_d} z^\top p + \sum_{i \in \cI} \psi_i(p) = \max_{u \in \R^d} \sup_{\sum_{i \in \cI} y_i = u -z}  - \sum_{i \in \cI} \psi^*_i(y_i) + \minco{u}{i} \,.
    \end{equation*}
    Fixing $p^* \in \argmin_{p \in \Delta_d} z^\top p + \psi(p)$ and $u^* \in \argmax_{u \in \R^d} - \psi^*(u - z) + \minco{u}{i}$, 
    \Cref{lem:strong-duality} also implies that
    \begin{equation*}
        u^* - z \in \partial \psi(p^*) \,.
    \end{equation*}
    Hence, by \citep[Proposition 5.4.6]{bertsekas2009} and the assumption that $\cap_{i \in \cI} \relint(\dom(\psi_i)) \neq \emptyset$, there exists a collection of vectors $(g_i)_{i \in \cI}$ each satisfying $g_i \in \partial \psi_i(p^*)$ such that $\sum_{i \in \cI} g_i = u^* - z$. 
    We then get that
    \begin{align*}
        \psi^*(u^* - z) &= (u^* - z)^\top p^* - \psi(p^*) \\
        &= (\summ_{i \in \cI} g_i)^\top p^* - \psi(p^*) \\
        &= \sum_{i \in \cI} \brb{g_i^\top p^* - \psi_i(p^*)} \\
        &= \sum_{i \in \cI} \psi^*_i(g_i) \,,
    \end{align*}
    where the first and last equalities follow from \citep[Theorem 23.5]{rockafellar1970}.
    Combined with \eqref{eq:strong-duality-closed-sum}, this entails that
    \begin{equation} \label{eq:strong-duality-sum-subgradients}
        \min_{p \in \Delta_d} z^\top p + \psi(p) =  - \sum_{i \in \cI} \psi^*_i(g_i) + \minco{u^*}{i}  \,.
    \end{equation}

    Next, \Cref{lem:essentially-smooth-sum} gives that $\psi$ is essentially smooth when every $\psi_i$ is.
    If this condition holds, \Cref{lem:strong-duality} would additionally give that
    $u^*$ is unique and that $p^* \in \Delta_d \cap \interior(\cap_{i \in \cI}\dom(\psi_i))$.
    The latter implies that $p^* \in \interior(\dom(\psi_i))$ for every $i \in \cI$, which entails via Theorem 26.1 in \citep{rockafellar1970} that $\partial \psi_i(p^*) = \{\nabla \psi_i(p^*)\}$.
    Hence, $g_i$ in \eqref{eq:strong-duality-sum-subgradients} becomes $\nabla \psi_i(p^*)$ and $\sum_{i \in \cI} \nabla \psi_i(p^*) = u^* - z$.
    Finally, \Cref{lem:strong-duality} implies that $p^*$ is unique if it also holds that $\sum_i \psi_i$ is strictly convex in $ \interior(\cap_{i \in \cI}\dom(\psi_i))$.
\end{proof}

For a convex function $f \colon \R \rightarrow (-\infty, +\infty]$ and a vector $h \in \R^d$, we denote by $f \circ h^\top \colon \R^d \rightarrow (-\infty, +\infty]$ the function given by
\begin{equation*}
    (f \circ h^\top)(x) \coloneqq f(h^\top x) \,.
\end{equation*}
Whereas the function $h \circ f \colon \R^d \rightarrow (-\infty, +\infty]$ is given by
\begin{equation*}
    (h \circ f)(x) \coloneqq \inf_{x = \alpha h} f(\alpha) \,.
\end{equation*}

\begin{lemma} \label{lem:conjugate-linear-comp}
    For a proper convex function $f \colon \R \rightarrow (-\infty, +\infty]$ and a non-zero vector $h \in \R^d$, it holds that
    \begin{equation*}
        (f \circ h^\top)^* = h \circ f^* \,.
    \end{equation*}
\end{lemma}
\begin{proof}
    Note that any vector $y \in \R^d$ satisfies $y = \alpha h + y'$ for some $\alpha \in \R$ and $y' \in \R^d$ such that $h^\top y' = 0$.
    Then, starting from the definition of the conjugate, we have that 
    \begin{align*}
        (f \circ h^\top)^* (y) 
        &= \sup_{x \in \R^d} x^\top y - f(h^\top x) \\
        &= \sup_{x \in \R^d} x^\top (\alpha h + y') - f(h^\top x) \\
        &= \sup_{\beta \in \R,\, x' \in \R^d \colon h^\top x'=0} (\beta h + x')^\top (\alpha h + y') - f(h^\top (\beta h + x')) \\
        &= \sup_{\beta \in \R,\, x' \in \R^d \colon h^\top x'=0} \alpha \beta \norm{h}^2 + x'^\top y' - f( \beta \norm{h}^2) \\
        &= f^*(\alpha) + \sup_{x' \in \R^d \colon h^\top x'=0} x'^\top y'\\
        &= \inf_{y = \alpha h} f^*(\alpha) = (h \circ f^*)(y) \,,
    \end{align*}
    where the penultimate equality uses that the supremum is $0$ if $y$ and $h$ are collinear and
    $+\infty$ otherwise.
    In the latter case $\{\alpha \in \R \colon y = \alpha h\}$ is empty and $(h \circ f^*)(y)$ is vacuously $+\infty$, while in the former, said set is a singleton.
\end{proof}

\begin{lemma} \label{lem:subgradient-linear-comp}
    For a proper convex function $f \colon \R \rightarrow (-\infty, +\infty]$ and a non-zero vector $h \in \R^d$, it holds for every $x \in \R^d$ that
    \begin{equation*}
        \partial (f \circ h^\top)(x) = h \partial f(h^\top x) \,.
    \end{equation*}
\end{lemma}
\begin{proof}
    This follows directly from Theorem 23.9 in \citep{rockafellar1970} since the range of $h^\top$ is $\R$, and so, it automatically has a non-empty intersection with $\relint(\dom(f))$ \citep[Theorem 6.2]{rockafellar1970}. 
\end{proof}

\begin{lemma} \label{lem:essentially-smooth-linear-comp}
    Let $f \colon \R \rightarrow (-\infty, +\infty]$ be a proper and closed convex function, and let $h \in \R^d$ be a non-zero vector. 
    If $f$ is essentially smooth, then so is $f \circ h^\top$.
\end{lemma}
\begin{proof}
    Firstly, $f \circ h^\top$ is also a closed convex function via Proposition 1.1.4 in \citep{bertsekas2009}.
    It is also proper since the range of $h^\top$ is the entirety of $\R$.
    Theorem 26.1 in \citep{rockafellar1970} states that a proper and closed convex function is essentially smooth if and only if the sub-differential at any point contains at most one element.
    Since $f$ is essentially smooth and, by \Cref{lem:subgradient-linear-comp}, every $x \in \R^d$ satisfies
    \[
        \partial (f \circ h^\top)(x) = h \partial f(h^\top x) \,,
    \]
    it follows that $(f \circ h^\top)$ too is essentially smooth.
\end{proof}

\begin{lemma} \label{lem:strictly-convex-sum-linear-comp}
    For every $i$ in a finite index set $\cI$, let $f_i \colon \R \rightarrow (-\infty, +\infty]$ be a proper convex function and $h_i \in \R^d$ be a non-zero vector.
    Assume that $\interior(\cap_{i \in \cI} \dom(f_i \circ h_i^\top)) \neq \emptyset$ and that $\spn\{h_i\}_{i \in \cI} = \R^d$.
    Then, if every $f_i$ is strictly convex in $\interior(\dom(f_i))$, it holds that $\sum_{i \in \cI} f_i \circ h_i^\top$ is strictly convex in $\interior(\cap_{i \in \cI} \dom(f_i \circ h_i^\top))$.
\end{lemma}
\begin{proof}
    Note that $\dom(f_i \circ h_i^\top) = \{x \in \R^d \colon h_i^\top x \in \dom(f_i)\}$. Define $\psi \coloneqq \sum_{i \in \cI} f_i \circ h_i^\top$. Then, $\dom(\psi) = \cap_{i \in \cI} \dom(f_i \circ h_i^\top)$ and $\psi$ is a proper convex function.
    Note, moreover, that $x \mapsto h_i^\top x$ maps vectors in $\interior(\dom(f_i \circ h_i^\top))$ to $\interior(\dom(f_i))$.
    Now, for any $\alpha \in (0,1)$ and $x, y \in \interior(\dom(\psi))$ such that $x \neq y$, we have that
    \begin{align*}
        \psi(\alpha x + (1-\alpha) y) &= \sum_{i \in \cI} f_i(h_i^\top (\alpha x + (1-\alpha) y)) \\
        &= \sum_{i \in \cI} f_i(\alpha h_i^\top x + (1-\alpha) h_i^\top y) \\
        &< \sum_{i \in \cI} \alpha f_i(h_i^\top x) + (1-\alpha) f_i(h_i^\top y) \\
        &= \alpha \psi(x) + (1-\alpha) \psi(y) \,,
    \end{align*}
    where the inequality holds since $h_i^\top x, h_i^\top y \in \interior(\dom(f_i \circ h_i^\top))$ for all $i \in \cI$ and $h_i^\top x \neq h_i^\top y$ for at least one $i \in \cI$ since $x - y \neq 0$ and $\spn\{h_i\}_{i \in \cI} = \R^d$. We conclude then that $\psi$ is strictly convex in $\interior(\dom(\psi))$ as required.
\end{proof}

\begin{lemma} \label{lem:strong-duality-sum-one-d}
    Fix a vector $z \in \R^d$ and a finite index set $\cI$.
    For every $i \in \cI$, let $h_i \in \R^d$ be a non-zero vector and $f_i \colon \R \rightarrow (-\infty, +\infty]$ a proper and closed convex function. 
    Assume that $\relint(\Delta_d) \cap \brb{ \cap_{i \in \cI} \relint(\{ x \in \R^d \colon h_i^\top x \in \dom(f_i)\}) } \neq \emptyset$.
    Then, it holds that 
    \begin{equation*}
        \min_{p \in \Delta_d} z^\top p + \sum_{i \in \cI} f_i(h_i^\top p) = \max_{u \in \R^d} \sup_{\sum_{i \in \cI} \alpha_i h_i = u - z}  - \sum_{i \in \cI}  f_i^*(\alpha_i) + \minco{u}{i} \,,
    \end{equation*}
    and for any $p^* \in \Delta_d$ minimizing the L.H.S. and $u^* \in \R^d$ maximizing the R.H.S.\,,
    \begin{equation*}
        \min_{p \in \Delta_d} z^\top p + \sum_{i \in \cI} f_i(h_i^\top p) =  - \sum_{i \in \cI} f^*_i(\alpha^*_i) + \minco{u^*}{i} \,,
    \end{equation*}
    for some $(\alpha^*_i)_{i \in \cI}$ satisfying $\alpha^*_i \in \partial f_i(h_i^\top p^*)$ and $\sum_{i \in \cI} \alpha^*_i h_i = u^* - z$.
    If, additionally, every $f_i$ is essentially smooth, then $u^*$ is unique, $p^* \in \Delta_d \cap \interior(\cap_{i \in \cI} \{ x \in \R^d \colon h_i^\top x \in \dom(f_i)\})$, $\sum_{i \in \cI} f'_i(h_i^\top p^*) h_i = u^* - z $, and 
    \begin{equation*}
        \min_{p \in \Delta_d} z^\top p + \sum_{i \in \cI} f_i(h_i^\top p) =  - \sum_{i \in \cI} f^*_i(f'_i(h_i^\top p^*))  + \minco{u^*}{i} \,.
    \end{equation*}
    If, moreover, $\spn\{h_i\}_{i \in \cI} = \R^d$ and every $f_i$ is strictly convex on $\interior(\dom(f_i))$, then $p^*$ too is unique.
\end{lemma}
\begin{proof}
    Define $\psi_i \coloneqq f_i \circ h_i^\top$.
    Notice then that $\dom(\psi_i) = \{ x \in \R^d \colon h_i^\top x \in \dom(f_i)\}$.
    Moreover, $\psi_i$ is a closed and proper convex function.
    Hence, our assumption that $\relint(\Delta_d) \cap \brb{ \cap_{i \in \cI} \relint(\{ x \in \R^d \colon h_i^\top x \in \dom(f_i)\}) } \neq \emptyset$ implies via \Cref{lem:strong-duality-sum} that
    \begin{equation*}
        \min_{p \in \Delta_d} z^\top p + \sum_{i \in \cI} \psi_i(p) = \max_{u \in \R^d} \sup_{\sum_{i \in \cI} y_i = u -z}  - \sum_{i \in \cI} \psi^*_i(y_i) + \minco{u}{i} \,.
    \end{equation*}
    Using this, the definition of $\psi_i$, and \Cref{lem:conjugate-linear-comp}, we obtain that
    \begin{align*}
        \min_{p \in \Delta_d} z^\top p + \sum_{i \in \cI} f_i(h_i^\top p) &= \max_{u \in \R^d} \sup_{\sum_{i \in \cI} y_i = u -z} - \sum_{i \in \cI} \inf_{y_i = \alpha_i h_i} f_i^*(\alpha_i)   + \minco{u}{i} \\
        &= \max_{u \in \R^d} \sup_{\sum_{i \in \cI} \alpha_i h_i = u - z }  - \sum_{i \in \cI}  f_i^*(\alpha_i) + \minco{u}{i}\,.
    \end{align*}
    For any 
    \begin{align*}
        &p^* \in \argmin_{p \in \Delta_d} z^\top p + \sum_{i \in \cI} f_i(h_i^\top p) \\ 
        &\hspace{8em}\text{and} \quad u^* \in \argmax_{u \in \R^d} \sup_{\sum_{i \in \cI} \alpha_i h_i = u - z }  - \sum_{i \in \cI}  f_i^*(\alpha_i) + \minco{u}{i}\,,
    \end{align*}
    \Cref{lem:strong-duality-sum} also gives that
    \begin{equation*}
        \min_{p \in \Delta_d} z^\top p + \sum_{i \in \cI} f_i(h_i^\top p) =  - \sum_{i \in \cI} \psi^*_i(g_i) + \minco{u^*}{i} \,,
    \end{equation*}
    for some $(g_i)_{i \in \cI}$ satisfying $g_i \in \partial \psi_i(p^*)$ and $\sum_{i \in \cI} g_i = u^* - z$.
    On the other hand, the definition of $\psi_i$ and \Cref{lem:subgradient-linear-comp} give that 
    \[\partial \psi_i(p^*) = \partial(f_i \circ h_i^\top)(p^*) = h_i \partial f_i(h_i^\top p^*) \,,\]
    implying the existence of real weights $(\alpha^*_i)_{i \in \cI}$ such that $g_i = \alpha^*_i h_i$ and $\alpha^*_i \in \partial f_i(h_i^\top p^*)$.
    Hence, \Cref{lem:conjugate-linear-comp} gives that
    \[
        \psi^*_i(g_i) = \psi^*_i(\alpha^*_i h_i) = \inf_{\alpha^*_i h_i = \alpha h_i} f_i^*(\alpha) = f_i^*(\alpha^*_i) \,,
    \]
    using in the end that $h_i \neq \zeros$.
    We have thus shown that
    \begin{align}\label{eq:strong-duality-sum-one-d-subgradients}
        \min_{p \in \Delta_d} z^\top p + \sum_{i \in \cI} f_i(h_i^\top p) =  - \sum_{i \in \cI} f^*_i(\alpha^*_i) + \minco{u^*}{i}\,,
    \end{align}
    for some $(\alpha^*_i)_{i \in \cI}$ satisfying $\alpha^*_i \in \partial f_i(h_i^\top p^*)$ and $\sum_{i \in \cI} \alpha^*_i h_i = u^* - z $.

    Next, if $f_i$ is essentially smooth, then so is $f_i \circ h_i^\top$ by \Cref{lem:essentially-smooth-linear-comp}.
    Under that condition, \Cref{lem:strong-duality-sum} gives that $u^*$ is unique and that $p^* \in \Delta_d \cap \interior(\cap_{i \in \cI} \{ x \in \R^d \colon h_i^\top x \in \dom(f_i)\})$.
    The latter implies that $p^* \in \interior(\{ x \in \R^d \colon h_i^\top x \in \dom(f_i)\})$ for every $i \in \cI$, which entails that $h_i^\top p^* \in \interior(\dom(f_i))$.
    The essential smoothness of $f_i$ then entails via Theorem 26.1 in \citep{rockafellar1970} that $\partial f_i(h_i^\top p^*) = \{f_i'(h_i^\top p^*)\}$.
    Hence, $\alpha^*_i$ in \eqref{eq:strong-duality-sum-one-d-subgradients} becomes $f_i'(h_i^\top p^*)$ and we get that $\sum_{i \in \cI} f_i'(h_i^\top p^*) h_i = u^* - z $.
    Finally, \Cref{lem:strictly-convex-sum-linear-comp} gives that $\sum_i \psi_i = \sum_i f_i \circ h_i^\top$ is strictly convex in the interior of its domain whenever it holds that $\spn\{h_i\}_{i \in \cI} = \R^d$ and every $f_i$ is strictly convex in the interior of its domain (which is not empty via essential smoothness).
    Thus, \Cref{lem:strong-duality-sum} implies under these conditions that $p^*$ is also unique.
\end{proof}

For a function $f \colon \R^d \rightarrow (-\infty, +\infty]$ that is differentiable in the interior of its domain, the Bregman divergence with respect to $f$ of $x \in \R^d$ from $y \in \interior(\dom(f)$ is defined as
\begin{equation*}
    D_{f}(x\|y) \coloneqq f(x) - f(y) - \inprod{\nabla f(y), x - y} \,.
\end{equation*}

\begin{proposition}\label{prop:general-variance-generic}
    Fix a finite index set $\cI$ and associate to every $i \in \cI$ a non-zero vector $h_i \in \R^d$ and a weight $\eta_i > 0$.
    Let $f \colon \R \rightarrow (-\infty, +\infty]$ be a convex function that is proper, closed, and essentially smooth. %
    Assume also that
    $\relint(\Delta_d) \cap \brb{ \cap_{i \in \cI} \relint(\{ x \in \R^d \colon h_i^\top x \in \dom(f)\}) } \neq \emptyset$.
    Fix $y, z \in \R^d$ and let
    \[
        p^* \in \argmin_{p \in \Delta_d} y^\top p + \sum_{i \in \cI} \frac{1}{\eta_i}f(h_i^\top p) \,.
    \]
    Then, 
    \begin{multline*}
        (y+z)^\top p^* + \sum_{i \in \cI} \frac{1}{\eta_i} f(h_i^\top p^*) - \min_{p \in \Delta_d} \bbcb{ (y+z)^\top p + \sum_{i \in \cI} \frac{1}{\eta_i}f(h_i^\top p) } \\\leq \inf_{c,\{\alpha_i\}_{i \in \cI} \colon \sum_{i \in \cI} \alpha_i h_i = z - c \ones  }\sum_{i \in \cI} \frac{1}{\eta_i} D_{f^*}\brb{f'(h_i^\top p^*) - \eta_i \alpha_i \big\|f'(h_i^\top p^*)} \,.
    \end{multline*}
\end{proposition}
\begin{proof}
    Fix any $p \in \Delta_d$ and constant $c \in \R$. Since $\inprod{c \ones, p^* - p} = 0$, we get that
    \begin{align*}
        &(y+z)^\top p^* + \sum_{i \in \cI} \frac{1}{\eta_i}f(h_i^\top p^*) - (y+z)^\top p - \sum_{i \in \cI} \frac{1}{\eta_i}f(h_i^\top p) \\
        &\hspace{8em}= (y+z)^\top (p^*-p) + \sum_{i \in \cI} \frac{1}{\eta_i}f(h_i^\top p^*)  - \sum_{i \in \cI} \frac{1}{\eta_i}f(h_i^\top p) \\
        &\hspace{8em}= (y+z-c\ones)^\top (p^*-p) + \sum_{i \in \cI} \frac{1}{\eta_i}f(h_i^\top p^*)  - \sum_{i \in \cI} \frac{1}{\eta_i}f(h_i^\top p) \,.
    \end{align*}
    It can be readily seen that $f_i \coloneqq (1/\eta_i)f$ satisfies the same properties of $f$ that are listed in the statement of the proposition.
    Hence, \Cref{lem:strong-duality-sum-one-d} gives that
    \begin{align*}
        y^\top p^* + \sum_{i \in \cI} \frac{1}{\eta_i}f(h_i^\top p^*) &= - \sum_{i \in \cI} f^*_i(f'_i(h_i^\top p^*)) + \minco{u_{p^*}}{i}\\
        &= - \sum_{i \in \cI} \frac{1}{\eta_i} f^*(\eta_i f'_i(h_i^\top p^*)) + \minco{u_{p^*}}{i}\\
        &= - \sum_{i \in \cI} \frac{1}{\eta_i} f^*( f'(h_i^\top p^*)) + \minco{u_{p^*}}{i}
    \end{align*}
    for some $u_{p^*} \in \R^d$, which satisfies that $\sum_{i \in \cI} (1/\eta_i) f'(h_i^\top p^*) h_i = u_{p^*} -y $.
    \Cref{lem:strong-duality-sum-one-d} also implies that $p^* \in \interior(\cap_{i \in \cI} \{ x \in \R^d \colon h_i^\top x \in \dom(f)\})$, thus, as a consequence, $h_i^\top p^* \in \interior(\dom(f))$ for every $i \in \cI$.
    At the same time, another implication of \Cref{lem:strong-duality-sum-one-d} is that 
    \begin{align*}
        &- (y+z-c\ones)^\top p  - \sum_{i \in \cI} \frac{1}{\eta_i}f(h_i^\top p) \\
        &\hspace{6em} \leq \min_{u \in \R^d} \inf_{\sum_{i \in \cI} \alpha_i h_i = u -z - y + c \ones}   \sum_{i \in \cI}  f_i^*(\alpha_i) - \minco{u}{i} \\
        &\hspace{6em}\leq - \minco{u_{p^*}}{i} +
        \inf_{\sum_{i \in \cI} \alpha_i h_i = u_{p^*} -z - y + c \ones} \sum_{i \in \cI}  f_i^*(\alpha_i) \\
        &\hspace{6em}= - \minco{u_{p^*}}{i} +
        \inf_{\sum_{i \in \cI} \alpha_i h_i = u_{p^*} -z - y + c \ones} \sum_{i \in \cI} \frac{1}{\eta_i}  f^*(\eta_i \alpha_i) \\
        &\hspace{6em}= - \minco{u_{p^*}}{i} +
        \inf_{\sum_{i \in \cI} \alpha_i h_i = -z + c \ones} \sum_{i \in \cI} \frac{1}{\eta_i}  f^*(f'(h_i^\top p^*) + \eta_i \alpha_i) \\
        &\hspace{6em}= - \minco{u_{p^*}}{i} +
        \inf_{\sum_{i \in \cI} \alpha_i h_i = z - c \ones} \sum_{i \in \cI} \frac{1}{\eta_i}  f^*(f'(h_i^\top p^*) - \eta_i \alpha_i) \,,
    \end{align*}
    where the third equality uses that $\sum_{i \in \cI} (1/\eta_i) f'(h_i^\top p^*) h_i = u_{p^*} -y $.

    Combining the results derived thus far we get that
    \begin{align*}
        &(y+z)^\top p^* + \sum_{i \in \cI} \frac{1}{\eta_i}f(h_i^\top p^*) - (y+z)^\top p - \sum_{i \in \cI} \frac{1}{\eta_i}f(h_i^\top p) \\
        &\:\leq (z - c\ones)^\top p^* + \inf_{\sum_{i \in \cI} \alpha_i h_i = z - c \ones} \sum_{i \in \cI} \frac{1}{\eta_i}  f^*(f'(h_i^\top p^*) - \eta_i \alpha_i) - \frac{1}{\eta_i} f^*( f'(h_i^\top p^*)) \\
        &\: = \inf_{\sum_{i \in \cI} \alpha_i h_i = z - c \ones}  (z - c\ones)^\top p^* + \sum_{i \in \cI} \frac{1}{\eta_i}  f^*(f'(h_i^\top p^*) - \eta_i \alpha_i) - \frac{1}{\eta_i} f^*( f'(h_i^\top p^*)) \\
        &\: = \inf_{\sum_{i \in \cI} \alpha_i h_i = z - c \ones}  \sum_{i \in \cI} \alpha_i h_i^\top p^* + \sum_{i \in \cI} \frac{1}{\eta_i}  f^*(f'(h_i^\top p^*) - \eta_i \alpha_i) - \frac{1}{\eta_i} f^*( f'(h_i^\top p^*)) \\
        &\: = \inf_{\sum_{i \in \cI} \alpha_i h_i = z - c \ones}   \sum_{i \in \cI} \frac{1}{\eta_i} \Brb{ \eta_i \alpha_i h_i^\top p^* +  f^*(f'(h_i^\top p^*) - \eta_i \alpha_i) -  f^*( f'(h_i^\top p^*)) }\\
        &\: = \inf_{\sum_{i \in \cI} \alpha_i h_i = z - c \ones}   \sum_{i \in \cI} \frac{1}{\eta_i} \Brb{ \eta_i \alpha_i {f^{*}}'(f'(h_i^\top p^*)) +  f^*(f'(h_i^\top p^*) - \eta_i \alpha_i) -  f^*( f'(h_i^\top p^*)) }\\
        &\: = \inf_{\sum_{i \in \cI} \alpha_i h_i = z - c \ones}   \sum_{i \in \cI} \frac{1}{\eta_i} D_{f^*}\brb{f'(h_i^\top p^*) - \eta_i \alpha_i \big\|f'(h_i^\top p^*)} \,,
    \end{align*}
    where the fourth equality follows from
    Theorem 26.5 in \citep{rockafellar1970} and the fact that $h_i^\top p^* \in \interior(\dom(f))$ for every $i \in \cI$.
    The sought result now follows after taking the infimum over $c$ and the minimum over $p$.
\end{proof}

\end{document}

%% file: 1_Introduction.tex
Consider the following recommendation problem:
An ad placement platform has a catalogue of products\textemdash let's say cars for concreteness\textemdash and is called to make product recommendations to a stream of users connecting to the system.
This catalogue may have an innate structure:
for example, cars may be classified
according to brand (Ford, Chevrolet, Toyota, Peugeot,\dots),
type (SUVs, compacts, minivans,\dots),
propulsion (gas, hybrid, or electric),
and/or any other relevant trait.
In turn, this categorization across different traits (brand, type, propulsion, \dots) may have a commensurate effect on the value of each product:
thus, a user looking to buy a large family car might have a relatively high value for all mid-size cars in the catalogue, a somewhat lower value for SUVs, and a far lower value for compacts;
and, depending on the user's criteria on the catalogue's other traits, this could further constrain the value range of each combination of traits.

Seen as a repeated decision problem, the learner here seeks to minimize their regret in the presence of a large number of distinct alternatives, some of which may be more ``similar'' than others.
These similarities are defined by means of a family of traits which group together alternatives that share a given attribute (or more).
More precisely, each trait gives rise to a partition of the set of possible choices, and each partition (or combination thereof) might greatly constrain the range of rewards that can be observed within:
for instance, %
the value of two hybrid SUVs of the same make and model but different year would be highly correlated;
by contrast, the rewards associated to a hybrid compact and a hybrid SUV of different make could diverge considerably.
A natural goal here is to determine the degree to which such innate similarities within the set of alternatives may lead to improved regret guarantees, and how to achieve them.

A natural way to operationalize such a structure is to encode it as a hierarchy, \ie a tree-structured notion of similarity over the alternatives.
This presupposes fixing a priority ranking over the traits; e.g., 1) brand, 2) type, 3) propulsion, etc.
Then, the traits from the first up to, say, the $n$-th collectively partition the alternatives, with each set of this partition corresponding to a possible combination of domain values that this group of traits can take.
This partition in turn refines the one induced by the traits from $1$ up to $n-1$.
Returning to our example, the combination of ``brand'' and ``type'' partitions the set of alternatives (here, cars) into ``Ford sedans'', ``Ford SUVs'', ``Toyota sedans'', etc; which is a refinement of the partition induced by ``brand'' alone: ``Ford'', ``Toyota'', etc.
Ultimately, one obtains a nested sequence of partitions of the set of alternatives, which can be represented as a tree whose root is the entire set of alternatives and whose leaves correspond to individual alternatives.
See \Cref{app:similarities} for a more formal exposition of this process.

Crucially, the difference in outcome between two alternatives (in terms of gains or losses) can now be simply bounded as a function of the height of their lowest common ancestor in the tree; the lower it is, the more traits they agree on.
Motivated by this example, and taking a more abstract point of view, we study in this work a multi-armed bandit problem where the actions are the leaves of a given tree, and where the losses of any two actions can diverge by no more than a given function of the height of their most recent common ancestor.
The main question we seek to answer is: \emph{when} can these similarity constraints translate into strictly improved regret guarantees, and \emph{what feedback is needed} to actually exploit them?

The first point we establish is that, perhaps surprisingly, exploiting this structure is not possible in general under standard 1-point bandit feedback.
This is in contrast to parametric bandit models like linear bandits \citep{bubeck2012}, where one can improve upon the minimax regret rate of unstructured bandits.
We prove this impossibility result by building upon the fact that one cannot exploit the effective range of the losses in (adversarial) multi-armed bandits \citep{gerchinovitz2016refined}; that is, the maximum possible difference between the losses of two actions.

It is no surprise then that the closely related \emph{nested bandit} model of \citet{martin2022nested} assumes access to more informative feedback.
In that work, they also consider a tree-based representation with leaves as actions.
They assume that the loss of an action admits a nested decomposition along the path to the root of the tree, with the expectation that the intrinsic losses are smaller for nodes at lower levels as they represent less relevant characteristics of the action.
They adopt a semi-bandit feedback model: upon choosing an action, all 
node-level losses
along the path to the root are revealed.
The regret bounds they achieve depend on an \emph{effective number of actions} $K_\mathrm{eff}$ (which is a function of the importance and size of intermediate node sets) rather than on the total number of actions $K$.
However, in many applications, this semi-bandit feedback is difficult to justify; in the car-catalogue example, one typically observes only an overall satisfaction signal for the recommended car, not separate losses attributable to brand, model, color, etc.
As a second limitation, they only focus on adversarial losses and do not provide guarantees in the stochastic setting.

Here, we address both limitations.
We develop a best-of-both-worlds approach that yields guarantees in both stochastic and adversarial regimes depending on notions of effective number of actions, and we substantially weaken the feedback requirements by moving to multi-point protocols\textendash culminating in a minimal two-point feedback model in which only the losses of two carefully coupled actions are observed at each round.
Compared to semi-bandit feedback, it is often more realistic to obtain such multi-point feedback\textendash for instance from several users making different choices. Likewise, in routing and path-finding applications, one may send several drivers along different routes and observe their corresponding total travel times.
Importantly, our feedback model does not provide free exploration; the losses of all chosen actions are accounted for in the regret.

In the final part of this work, we show that, beyond the tree-structured model, multi-point feedback can also be advantageous in the more familiar Lipschitz bandits problem \citep{kleinberg2008multi}.
In particular, under two-point feedback, we obtain new regret guarantees that, in lower dimensions, are significantly better than what can be achieved under one-point feedback.
These new guarantees are obtained by applying our techniques to a hierarchical discretization of the action space.

We now summarize our contributions.

\subsection{Contributions}

Our contributions can be summarized as showing that action similarities can be leveraged, under a relatively weak feedback model, to obtain  strictly sharper best-of-both-worlds bounds compared to standard bandit guarantees.

\paragraph{(1) A lower bound for one-point bandit feedback.}
We prove that one-point (bandit) feedback cannot, in general, exploit tree-induced similarity constraints: for any such similarity structure, there always exist loss sequences that force $\Omega(\sqrt{KT})$ regret. This motivates considering feedback models that reveal information from at least two actions per round.

\paragraph{(2) Best-of-both-worlds guarantees under tree-induced similarities.}
We adopt a follow-the-regularized-leader algorithmic template with a tree-aware regularizer that achieves logarithmic-type guarantees in stochastic regimes and $\sqrt{T}$-type guarantees in adversarial regimes.
In both cases, the similarity structure yields improvements over structure-agnostic baselines by replacing the ambient dependence on $K$ with an \emph{effective} number of actions $K_{\mathrm{eff}}$ (and, in the stochastic case, an effective gap-dependent complexity), which can be dramatically smaller than $K$ when losses vary smoothly across the action set.

\paragraph{(3) Multi-point feedback relaxations, up to a minimal two-point protocol.}
We consider multi-point bandit protocols ($m \ge 2$) ranging from an $(L+1)$-point scheme (where $L$ is the depth of the tree) aligned with the hierarchy down to a \emph{minimal} two-point feedback model in which only the losses of two carefully coupled actions are observed per round. We show that we can achieve guarantees comparable to those that can be obtained under the semi-bandit model of \citep{martin2022nested}.

\paragraph{(4) Application to Lipschitz bandits.}
As an application, we show that under two-point feedback, our techniques can yield interesting results in the Lipschitz bandits problem, where we show that it is possible to achieve $\sqrt{T}$ regret when $d \leq 2$.

\subsection{Related Work}

Multi-armed bandits are a standard framework for sequential learning under uncertainty and have been extensively studied. We refer to \cite{lattimore2020bandit} for a comprehensive overview and focus here on the most relevant works for our setting.

\paragraph{Exploiting structure in multi-armed bandits.}
 Classical bandit models typically assume a finite set of independent arms. However, many modern applications involve structured action sets in which arms are related through geometry, correlations, or combinatorial constraints. This has motivated a substantial body of work on bandits with structured action spaces.
Structured bandits extend the basic multi-armed bandit framework to large or infinite arm sets by imposing known relationships on the reward function across arms, which can be exploited to improve learning performance. Prominent examples include linear bandits \citep{abbasi2011improved,agrawal2013thompson,ito2020tight} and combinatorial bandits \citep{wen2015efficient,combes2015combinatorial}. A unifying treatment of several such structures in the stochastic setting was proposed by \cite{combes2017minimal}. Another important class is given by Lipschitz or Hölder bandits, where arms lie in a metric (possibly continuous) space and regret bounds depend on the covering or packing dimension of this space \citep{kleinberg2008multi,slivkins2011multi,kleinberg2019bandits,podimata2021adaptive}.

In contrast to these parametric or metric-based models, we consider a setting where the similarity structure of the action space is encoded by a tree. This representation is motivated both by Lipschitz bandits (see Sec.~\ref{sec:build-tree}) and by applications involving catalogues of objects that share hierarchical attributes (see Sec.~\ref{sec:preliminaries}). Such implicit tree structures have previously been studied by \cite{slivkins2011multi}, who refer to them as taxonomies of arms, in the stochastic setting, and by \cite{martin2022nested} for adversarial bandits with semi-bandit feedback.
Another relevant line of work is hierarchical stochastic bandits (e.g., \citealt{hong2022deep,hong2022,aouali23,nguyen2025}), where correlations between rewards are imposed via hierarchical Bayesian models.
\cite{hong2022deep}, in particular, consider a tree-structured action set where nearby actions can have strongly correlated rewards.
When the structure is favorable, they obtain better regret bounds compared to structure-agnostic algorithms.
However, they only consider a Bayesian setting and do not provide $\ln T$ bounds.

Finally, structural assumptions cannot always be fully exploited under standard bandit feedback. In particular, \cite{gerchinovitz2016refined} show that adaptation to the intrinsic loss range is impossible. We extend this limitation to our tree-structured setting in Section~\ref{sec:lowerbound}, motivating our multi-point feedback model.

\paragraph{Multi-point bandit feedback.}
Several existing works have shown that improved regret guarantees can be obtained when the learner observes the loss at additional points at each round. Such improvements date back to \cite{agarwal2010optimal}, which improves the regret from $O(T^{3/4})$ to $O(\sqrt{T})$ compared to the original one-point feedback gradient descent method of \citep{flaxman2004online, kleinberg2004} for bandit convex optimization. Subsequent works \citep{duchi2015optimal, shamir2017optimal, lu2024adaptive} similarly demonstrate the benefits of two-point feedback for reducing the variance of gradient estimates or compete with non-stationary bandit environments. However, these approaches rely on gradient estimation techniques and assume convex losses, which is not the case in our setting. In the multi-armed bandit literature, observing an additional point has been studied by \cite{sani2014exploiting} to adapt to ``easy data'' and by \cite{thune2018adaptation} to adapt to the effective range of the losses, thereby circumventing the impossibility result of \cite{gerchinovitz2016refined} for one-point feedback.
Additionally, \cite{degenne2018} study a stochastic bandit problem with occasional (free) additional observations, which allows obtaining constant regret bounds.
In our case, we show that two-point feedback allows us to exploit the tree structure of the action set. 
However, unlike in \citep{sani2014exploiting}, \citep{thune2018adaptation}, and \citep{degenne2018}, this additional observation does not constitute free exploration; rather, it is explicitly accounted for in our regret definition, which is better motivated by practical considerations.

%% file: 2_Preliminaries.tex
\paragraph{A Tree-Based Action Set}
We consider an action set $\cA \coloneqq \{1,\dots,K\} $.
We will represent $\cA$ as the set of leaves of a given, fixed rooted tree $\cT$ of depth $L$.
We summarize now some tree-related notation.
For each level $j\in\{0,1,\dots,L\}$, let $V_j$ be the set of nodes at depth $j$, with
$V_0=\{v_0\}$ (the root) and $V_L=\cA$.
We define
$\cV \coloneqq \cup_{j=1}^L V_j$ and
$\overbar{\cV} \coloneqq \cup_{j=0}^L V_j$.
We use $\prtsym$ to denote the parenthood relation: $\prtof{v}{w}$ means that $w$ is a child of $v$.
This can only happen if $v\in V_{j-1}$ and $w\in V_j$ for some $j\in[L]$.
We assume that every node has at least one child until level $L$:
for any $j\in\{0,\dots,L-1\}$ and $v\in V_j$, there exists at least one $w\in V_{j+1}$ such that $\prtof{v}{w}$.
In particular, $\prtof{v_0}{w}$ holds for every $w\in V_1$.

Let $\ancsym$ denote the reflexive and transitive closure of $\prtsym$ (ancestry relation),
so that $\ancof{v}{w}$ means that $v$ is an ancestor of $w$ (and every node is its own ancestor). Additionally, the function $\lvlfun \colon \overbar{\cV} \rightarrow [L]\cup\{0\}$ yields for any $v \in \overbar{\cV}$ the unique level $j$ such that $v \in V_j$.
Moreover, for every $j \in [L] \cup \{0\}$, the function $\ancfun_j \colon \bigcup_{l \geq j}\cV_l \rightarrow \cV_j$ maps every node at level $j$ or lower to its unique ancestor at level $j$. 
The function $\prtfun \colon \cV \rightarrow \overbar{\cV}$ maps a node $v$ to its unique parent $\prt{v}$.
For each $j\in[L]\cup\{0\}$, we define the equivalence relation $\comanc_j$ on $\bigcup_{l \ge j}V_l$ by
\[
v \comanc_j w \quad \Longleftrightarrow \quad \anc{v}{j} = \anc{w}{j}.
\]
When restricted to leaves, $a\comanc_j b$ means that $a$ and $b$ lie in the same subtree rooted at $\anc{a}{j}$. We encode the hierarchical proximity between actions through their divergence level, using a \emph{similarity} function $\similarity:\cA \times \cA \rightarrow \{1, \dots, L+1\}$ defined as
\[
\similarity(a,b) \coloneqq 1 + \max\{ j\in\{0,\dots,L\} \mid a \comanc_j b \},
\qquad a,b\in\cA.
\]
Equivalently, for two actions $a, b \in \cA$, $\similarity(a,b)$ is the first level at which $a$ and $b$ diverge ($\similarity(a,a)=L+1$), so the larger $s(a,b)$, the more similar $a$ and $b$ are.
We identify functions $x:\cA\to\R$ with vectors in $\R^K$ through an arbitrary (fixed) bijection between $\cA$ and $[K]$.
The same applies to probability distributions over $\cA$, which we view as vectors in $\Delta_K$, the probability simplex in $\R^K$.
For any vector $x\in\R^K$ and any node $v\in\overbar{\cV}$, we define the subtree aggregate
\[
x[v] \coloneqq \sum_{a\in \cA:\ \ancof{v}{a}} x(a),
\]
\ie the mass (or value) of $x$ restricted to the subtree rooted at $v$.

\paragraph{Tree-compatible loss model}
This work is framed under a sequential online bandit problem (described in more detail later) in which the learner faces a sequence of loss functions $(y_t)_{t\ge 1}$ with $y_t:\cA\to \R$ generated by the environment.
We will assume that the loss functions respect the similarity between actions in the following sense: 
we assume there exist parameters $(\scale_j)_{j\in[L+1]}$ (with $\sigma_{L+1}=0$) such that for all rounds $t$ and all $a,b\in\cA$, \vspace*{-5pt}
\begin{equation}
\label{eq:smoothness condition}
\abs{y_t(a)-y_t(b)} \leq \scale_{s(a,b)}.
\end{equation} 
That is, pairs sharing a deeper common ancestor have more similar losses.
We posit, naturally, that $\scale_1 \ge \scale_2 \ge \dots \ge \scale_L$.
Additionally, for the similarity structure to provide an advantage to the learner, these values need to decrease sufficiently fast compared to the growth of the number of nodes in deeper levels of the tree.
The regret guarantees provided in the coming sections are naturally sensitive to this aspect.
Their precise dependence on the tree structure can serve as a guiding metric when constructing the tree starting from a more general similarity model.
In this work, %
we will mostly focus on the regret minimization aspect, taking the tree structure for granted.
The application of \Cref{sec:build-tree} is the exception to this, as there, we also undertake the task of designing the tree.

\paragraph{Online bandit game and multi-point feedback}
We now describe the sequential protocol and the regret notion used throughout the paper. 
Fix an integer $m$ between $1$ and $L+1$.
At every round $t$, the environment secretly draws a tree-compatible loss function $y_t$, and simultaneously, the learner plays $m\ge 1$ actions
$A_{t,1},\dots,A_{t, m}\in\cA$, observes $y_t(A_{t,1}),\dots,y_t(A_{t,m})$, and incurs their average $\frac{1}{m}\sum_{i=1}^m y_t(A_{t,i})$. 
The regret at horizon $T$ is 
\[
\reg{T}{m}
:=
\E\left[\sum_{t=1}^T \frac{1}{m}\sum_{i=1}^m y_t\!\big(A_{t, i}\big)\right]
-
\min_{a\in\cA}\ \E\left[\sum_{t=1}^T y_t(a)\right] \,,
\]
which the learner aims to minimize.
Note that the case $m=1$ coincides with the standard (one-point) bandit setting.
The case of $m \geq 2$ is considered in the light of the impossibility result presented in \Cref{sec:lowerbound}.
We assume that $y_t$ is independent of the learner's actions at round $t$ or later, conditioned on the random events up to the end of round $t-1$.
Besides this, we make no further assumption about the loss generating mechanism.
However, in \Cref{sec:semi-bandit-feedback,sec:bandit-feedback} we achieve improved guarantees when the loss sequence is benign in the following sense.
\begin{condition}[Generalized Stochastic Setting]\label{condition:stochastic}
    There exists a vector $\Delta \in \R^K_+$ satisfying $|\{a\in \cA \colon \Delta(a)=0\}| \leq 1$ such that $\reg{T}{m} \geq \E \sum_{t=1}^T \frac{1}{m}\sum_{i=1}^m\Delta(A_{t,i})$ holds for any algorithm.
\end{condition}
Conditions of this form are common in the literature (see e.g., \cite{zimmert2021tsallis,ito-bobw-graphs,bobw-blackbox}), in particular, this subsumes the standard stochastic setting where the losses are drawn from a fixed distribution at each round.

%% file: 5_Alt_LowerBound.tex
In the absence of any structure, the minimax regret of the bandit problem is $\Theta(\sqrt{KT})$ \citep{auer2002finite}.
We now show that also for our problem, an improvement of this rate is \emph{information-theoretically impossible} under one-point bandit feedback: 
no matter how small the $\sigma$ values are,
the worst-case regret remains $\Omega(\sqrt{KT})$.

Fix a depth-$L$ tree with leaf set $\cA$ and scale parameters $(\scale_j)_{j\in[L]}$. 
Define the class 
\[
\cY^T_\scale
\coloneqq
\Big\{
(y_t)_{t=1}^T :
\forall t,\ y_t:\cA\to[0,1],\ \forall a,b\in\cA,\ \abs{y_t(a)-y_t(b)} \le \scale_{\similarity(a,b)}
\Big\}.
\]
Recalling that $\sigma_L = \min_{j \in [L]} \sigma_j$, we now formally state the lower bound.

\begin{restatable}[One-point lower bound under tree-induced similarity]{reprop}{proplowerbound} \label{prop:one-point-tree-lb}
    Let $\rho\coloneqq \scale_L$. For $K>2$, $T>32(K-1)\log(14)$, and $\rho > 0.22\sqrt{(K-1)/T}$, any randomized one-point bandit algorithm satisfies
    \[
    \sup_{(y_t)_{t=1}^T \in \cY^T_\scale}\ \E\!\left[\reg{T}{1}(y_{1:T})\right]
    \;>\; \frac{1}{504}\sqrt{T(K-1)}.
    \]
    Equivalently, for any such algorithm there exists a tree-compatible loss sequence $(y_t)_{t=1}^T\in\cY^T_\scale$ inducing $\Omega(\sqrt{KT})$ expected regret.
\end{restatable}

\paragraph{Proof sketch.}
The argument is a direct reduction to the range-impossibility result of \citet[Corollary~4]{gerchinovitz2016refined}.
Let
\[
C_\rho \coloneqq \Big\{ x\in[0,1]^K : \max_{a,b\in\cA}\abs{x(a)-x(b)} \le \rho \Big\}
\qquad \text{and} \qquad
C_\rho^T \coloneqq \Big\{ (y_t)_{t=1}^T : \forall t,\ y_t\in C_\rho \Big\}.
\]
Any $x\in C_{\rho}$ satisfies $\abs{x(a)-x(b)} \le \rho = \scale_L \le \scale_{\similarity(a,b)}$, hence $C_{\rho}^T \subseteq \cY^T_\scale$.
Applying \citep[Corollary~4]{gerchinovitz2016refined} with $\rho=\scale_L$ yields the claimed $\Omega(\sqrt{KT})$ bound over $\cY^T_\scale$. $\hfill \blacksquare$

\paragraph{Intuition and impact on the sequel.}
A convenient way to interpret $C_\rho$ is via a ``baseline + perturbation'' representation:
each $x\in C_\rho$ can be written as $x(a)=b+u(a)$ where $b=\min_{a}x(a)\in[0,1-\rho]$ and $u(a)=x(a)-b\in[0,\rho]$.
With one-point feedback, the learner observes only a single absolute value $y_t(A_t)=b_t+u_t(A_t)$ per round and cannot disentangle the unknown baseline $b_t$
from the informative part $u_t(A_t)$; an adversary may drift $b_t$ over time without violating the range constraint, making the observations informationally ambiguous even when $\rho$ is tiny.
This motivates our focus on multi-point feedback: contrasts such as $y_t(A_{t, 1})-y_t(A_{t, 2})$ cancel the common baseline and expose the effective range, making the structure of the action set exploitable.

%% file: BuildingTheTree.tex
Under the same interaction model as before, suppose we are now given an arbitrary action set $\actions$ and a set $\losses$ of admissible loss functions, whose members map $\actions$ to $\R$.
We can always equip $\actions$ with the following metric:
\begin{equation} \label{def:infty-metric}
    d_{\actions,\infty}(a, b) \coloneqq \sup_{\loss \in \losses} |\loss(a) - \loss(b)| \:\: \forall a,b \in \cA \,. 
\end{equation}
To be precise, this function is a metric only if no two actions behave identically on all loss functions in $\losses$, which is fair to impose.
All members of $\losses$ are obviously $1$-Lipschitz with respect to $d_{\actions,\infty}$.

In the tree-based model we have studied thus far, $\losses$ was taken as the set of tree-compatible loss functions described in \Cref{sec:preliminaries}.
In particular, the induced metric in this case is given by $d_{\actions,\infty}(a, b) = \scale_{\similarity(a,b)}$, where, to recall, $\similarity(a,b) \coloneqq 1 + \max\{ j\in\{0,\dots,L\} \mid a \comanc_j b \}$. %
More than a metric, this function satisfies a strong form of the triangle inequality: $\scale_{\similarity(a,b)} \leq \max\{\scale_{\similarity(a,c)},\scale_{\similarity(c,b)}\}$ for all $a,b,c \in \actions$ since $\scale_1 \ge \dots \ge \scale_L$; hence, it is an \emph{ultrametric}.
Note that the tree structure can be recovered from this distance function if the similarity scales are strictly decreasing.
A similar notion of distance was studied by \cite{slivkins2011multi} in the stochastic setting, and was referred to as \emph{implicit distance}.
Seen from this angle, the algorithmic techniques presented so far are designed to leverage this special form of distance and the implied restriction on the losses via \eqref{def:infty-metric}.

We aim to illustrate here that our results can still be adapted to cases in which $d_{\actions,\infty}$ is a generic metric (or pseudometric).
This is done, naturally, by constructing a tree over the action set and forcing a specific decay schedule for the resulting similarity scales.
This constitutes a ``truncation'' or ``compression'' of the original metric structure, but still leads to interesting results.
We will focus henceforth on the two-point feedback setting.
In fact, the arguments used to prove the one-point feedback lower bound of \Cref{prop:one-point-tree-lb} can be easily adapted to show that it still holds in this more general setting provided that $\actions$ is finite and the conditions of the proposition hold with $\rho = \min_{a,b \in \actions \colon a \neq b} d_{\actions,\infty}(a,b)$.

Towards employing \Cref{alg:2-point-feedback} and its regret bound, we start by constructing a tree whose nodes are members of $\actions$.
Unlike before, the nodes in the last level of this tree need not include the entire set $\actions$ (which we do not require to be finite here).
We specify now two parameters to be used in the sequel: a constant $\treecst > 0$ and a positive integer $\nlvls$ representing the number of levels as before.
We will assume that $\treecst \geq \max_{a,b \in \actions} d_{\actions,\infty}(a,b)$.
For $\lvl \in [\nlvls]$, let $\actions_\lvl \subseteq \actions$ be a subset of actions that form a $\treecst 2^{-\lvl}$ covering with respect to $d_{\actions,\infty}$.\footnote{This means that $\min_{a' \in \actions_\lvl} d_{\actions,\infty}(a,a') \leq \treecst 2^{-\lvl}$ for all $a \in \actions$.}
For convenience, let $\actions_0 \coloneqq \{a_0\}$ with some arbitrary action $a_0$. 
Then, for any $a \in \actions_\lvl$, define its parent as its closest action in $\actions_{\lvl-1}$ with ties broken arbitrarily;
that is, $\prt{a} \in \argmin_{a' \in \actions_{\lvl-1}} d_{\actions,\infty}(a,a')$ with ties broken arbitrarily.
If one action belongs to multiple covers, each of these occurrences is to be regarded as a distinct node.

In this manner, it is possible that certain nodes in intermediate levels end up with no children.
Hence, we define a pruned version of the covering at each level $\lvl$:
\[
    \prunactions_\lvl = \{ a \in \actions_\lvl \mid \exists a' \in \actions_\nlvls \colon a = \anc{a'}{\lvl}  \} \,,
\]
implying that $\prunactions_\lvl \subseteq \actions_\lvl$ for each $\lvl $ and that $\prunactions_\nlvls = \actions_\nlvls$.
Now the sequence $(\prunactions_\lvl)_{\lvl= 0, \dots, \nlvls}$ forms a tree, which we will refer to as $\cT\brb{(\prunactions_\lvl)_{\lvl=0}^{\nlvls}}$, whose leaves are all at level $\nlvls$.
This tree induces a nested sequence of partitions of the set $\actions_\nlvls$, which will be the action set in the eyes of the algorithm.
For any pair of actions $a,b \in \actions_\nlvls$ such that $\similarity(a,b)=\lvl$,\footnote{To recall, this means that the lowest common ancestor of $a$ and $b$ is at level $\lvl-1$.} it holds that
\begin{align*}
     d_{\actions,\infty}(a,b) &\leq  d_{\actions,\infty}(a,\anc{a}{\lvl-1}) +  d_{\actions,\infty}(\anc{b}{\lvl-1},b) \\
    &\leq \sum_{k=\lvl-1}^{\nlvls-1} d_{\actions,\infty}(\anc{a}{k+1},\anc{a}{k}) + \sum_{k=\lvl-1}^{\nlvls-1} d_{\actions,\infty}(\anc{b}{k+1},\anc{b}{k}) \\
    &\le 2 \treecst \sum_{k=\lvl-1}^{\nlvls-1} 2^{-k} \le \treecst  2^{3 - \lvl} %
\end{align*}
where we have used the triangle inequality and the fact that $a \comanc_{\lvl-1} b$.
Hence, via \eqref{def:infty-metric}, setting $\scale_\lvl = \treecst  2^{3 - \lvl}$ is consistent with the assumption in \eqref{eq:smoothness condition}, applied now to the new action set $\actions_\nlvls$ and the tree $\cT\brb{(\prunactions_\lvl)_{\lvl=0}^{\nlvls}}$.
We can now instantiate the regret guarantee of \Cref{alg:2-point-feedback} on this tree structure, yielding the following result (the proof can be found in \Cref{app:lipschitz}).
We will focus here on the adversarial setting.

\begin{restatable}{recor}{corgeneral}\label{cor:corgeneral}
Suppose we run Algorithm~\ref{alg:2-point-feedback} on the tree $\cT\brb{(\prunactions_\lvl)_{\lvl=0}^{\nlvls}}$ with parameters $\psi_t, b_j$ as defined in Theorem~\ref{thm:2-point-feedback}, $\scale_\lvl = \treecst 2^{3 - \lvl} $, and $\delta_j \propto 2^{-2\lvl/3} |\actions_\lvl|^{1/3}$.
Then, assuming that $\min_{\lvl \in [\nlvls]}\delta_\lvl \geq 1/T $, it holds that
\[
        \reg{T}{2} \leq 
        16 \cdot (9 + 4 \sqrt{6}) \sqrt{\keff\brb{(\actions_\lvl)_{\lvl=0}^{\nlvls}} T} + \treecst \, 2^{-\nlvls} T \,,
\]
where 
$
    \keff\brb{(\actions_\lvl)_{\lvl=0}^{\nlvls}} \coloneqq \treecst^2 \lrb{\sum_{\lvl \in [\nlvls]} 2^{-2\lvl/3} {|\actions_\lvl|}^{1/3}}^3 \,.
$
\end{restatable}
The last term is a bound on the approximation error resulting from using $\actions_\nlvls$ as the action set.
This holds because $\actions_\nlvls$ is a $\treecst 2^{-\nlvls}$-cover of $\actions$.
This term can be turned into a small constant by choosing $\nlvls = \floor{\log_2 T}$, though in some cases balancing this term with the first can be beneficial.
When the cardinality of $\actions$ is finite, denoted by $K$ as before, then $\keff\brb{(\actions_\lvl)_{\lvl=0}^{\nlvls}}$ is never larger than $K$ up to small constants, regardless of the choice of $\nlvls$ and the covers $(\actions_\lvl)_{\lvl \in [\nlvls]}$.
Of course, in interesting examples, ones where the covers can be made small at coarser levels, $\keff\brb{(\actions_\lvl)_{\lvl=0}^{\nlvls}}$ can be made significantly smaller than $K$.
Hence, one should strive to ensure that the covers are near-optimal, which can be achieved by a simple greedy allocation strategy.
We elaborate more on this in \Cref{app:tree-near-optimal}, where we also argue that no other tree-building strategy results, via the bound of \Cref{thm:2-point-feedback}, in a significantly smaller $\keff$.

Now, let $\actions \subset \mathbb{R}^d$ be a compact set with diameter $D$ (\wrt to the Euclidean metric), and let $\losses$ be the set of $G$-Lipschitz loss functions from $\actions$ to $\R$.
It holds then that
\[
    d_{\actions,\infty}(a,b) \leq G \norm{a-b} \: \forall a,b \in \actions \,.
\]
In the ordinary interaction model---when the learner selects at each round a single action, observes the associated loss, and is evaluated through the regret $\reg{T}{1}$--- 
the minimax optimal regret rate is $\smash{\Omega(T^{(d+1)/(d+2)})}$, as shown by \cite{kleinberg2008multi}. 
Subsequent works have focused on adapting to the effective (zooming) dimension $d$ of the metric space \citep{slivkins2011multi, kleinberg2019bandits, podimata2021adaptive}, designing more adaptive algorithms \citep{heliou2021zeroth}, or improving robustness to corruption \citep{kang2023robust}. However, all these approaches rely on one-point feedback and therefore suffer from the above minimax rate in the worst case. Under a richer feedback model (``one-sided'' full-information) motivated by second-price auctions, \cite{cesa2017algorithmic} showed that this rate can be improved. 

We show now that under two-point feedback, \Cref{cor:corgeneral} can be used to obtain better regret guarantees compared to the one-point feedback minimax rate. 
We will optimize the depth of the tree depending on the dimension. For each $\lvl = 1,\dots,\nlvls$, let $\epsilon_\lvl = D\sqrt{d}\,2^{-\lvl}$ be a precision level, and let $\actions_\lvl$ be an $\epsilon_\lvl$-cover of $\cA$ (\wrt to the Euclidean metric) with minimal cardinality. 
Using a uniform grid, one can show that $|\actions_\lvl| \le (D\sqrt{d}/\epsilon_\lvl)^d = 2^{\lvl d}$.
Note that $\actions_\lvl$ is a $GD\sqrt{d} 2^{-\lvl}$-cover \wrt $d_{\actions,\infty}$; hence, we can choose $\treecst = GD\sqrt{d}$, which is an upper bound for $\max_{a,b \in \actions} d_{\actions,\infty}(a,b)$ as required.
The following result is then easy to obtain from \Cref{cor:corgeneral} (the proof can be found in \Cref{app:lipschitz}).
\begin{restatable}{recor}{corlipschitznew}\label{cor:lipschitznew}
Let $T \geq 8$. Algorithm~\ref{alg:2-point-feedback} run on the tree $\cT\brb{(\prunactions_\lvl)_{\lvl=0}^{\nlvls}}$ induced by the covers described above with $\nlvls = \floor{\frac{1}{d}\log_2(T) }$ and parameters $\psi_t, b_j$ as defined in Theorem~\ref{thm:2-point-feedback}, $\scale_\lvl = GD\sqrt{d}\,2^{3-\lvl}$ and $\delta_\lvl \propto 2^{\lvl (d-2)/3}$, satisfies the following regret upper bound %
\[
	\reg{T}{2} \lesssim GD \cdot \left\{  
	\begin{array}{ll}
		  \sqrt{T}   & \text{if } d = 1 \\ %
	    (\ln T)^{3/2} \sqrt{T}  & \text{if } d = 2 \\ %
		 \sqrt{d} T^{\frac{d-1}{d}} & \text{if } d >2 %
		\end{array} \right.  \,,
\]
where $\lesssim$ denotes an inequality up to a universal constant.
\end{restatable}

Interestingly, when $d \le 2$, the above result shows that two-point feedback is sufficient to recover (up to logarithmic factors for $d=2$) the optimal $\Omega(\sqrt{T})$ full-information rate. %
This follows from the fact that the effective number of arms induced by the underlying tree structure 
is $ O((\log T)^3)$
for $d\leq 2$, even though the total number of leaves can be as large as $T$.

%% file: App-Similarities.tex
In this appendix, we elaborate more on the motivating example given in the introduction.
In particular, we give more details on how a tree structure can be induced from this general setting.

\subsection{From traits to trees}
\input{Traits}

\subsection{A logical structure for the traits}

In this section, our aim is to elaborate on certain aspects of priority rankings and their compatibility with the innate structure of the family of traits.

\paragraph{Binary relations between traits}

To that end, we begin by considering certain natural binary relations among a family of traits $\traits = \{\trait_{\iTrait}\}_{\iTrait=1}^{\nTraits}$ on $\elems$.
The first such relation is that of \textit{refinement}:
given two traits $\trait, \traitalt \in \traits$, we say that $\trait$ is \emph{finer} than $\traitalt$\textemdash and we write $\trait \refines \traitalt$\textemdash if every element of $\trait$ is a subset of some element of $\traitalt$.
Another natural relation arises when alternatives are categorized according to different traits, in which case some traits might be subordinate to other more ``primitive'', higher-order traits that precede them in any chain of logical reasoning.
For instance, the \model of a car cannot be defined without implicitly specifying the car's \brand:
after all, an Impala is, by definition, a Chevrolet.

To capture this, we define a \emph{logical structure} on a family of traits $\traits$ as a binary relation $\requires$ (read: ``\emph{requires}'' or ``\emph{is a primitive of}'') satisfying the following conditions for all $\trait,\traitalt,\traitaltalt \in \traits$:
\begin{enumerate}
\item
\label[condition]{cond:irreflexivity}
\emph{Irreflexivity:}
$\trait\nrequires\trait$.
\item
\label[condition]{cond:transitivity}
\emph{Transitivity:}
If $\trait \requires \traitalt$ and $\traitalt \requires \traitaltalt$, then $\trait \requires \traitaltalt$.
\item
\label[condition]{cond:consistency}
\emph{Consistency:}
If $\trait \requires \traitalt$, then $\trait \refines \traitalt$.
\end{enumerate}
\smallskip
In short, \cref{cond:irreflexivity} states that a trait $\trait$ is not a primitive of itself,
while \cref{cond:transitivity} requires that primitives of primitives of $\trait$ are themselves primitives of $\trait$;
thus, taken together, \cref{cond:irreflexivity,cond:transitivity} simply say that $\requires$ is a strict partial order on $\traits$.
The last condition (consistency) posits that two alternatives that differ by a primitive of $\trait$ also differ by $\trait$ itself:
for example, since \brand is a primitive of \model, a Toyota and a Chevrolet can't be the same \model.
\footnote{%
This implication does not go the other way round:
in particular, if $\trait$ is finer than $\traitalt$, $\traitalt$ need not be a primitive of $\trait$.
For example, if $\elems = \{\text{a blue Ford}, \text{a yellow Chevy}, \text{a red Chevy}\}$ and $\trait_{\kolor}$ is a partition of $\elems$ by \kolor and $\trait_{\brand}$ is a partition by \brand, the partition $\trait_{\kolor}$ is (strictly) finer than $\trait_{\brand}$, even though \brand is not a primitive of \kolor:
in general, a car's \kolor does not determine its \brand, even though it happens to do so for this \emph{particular} set of cars.
This shows that the prerequisite relation $\requires$ is about the traits themselves as abstract entities, not about the set of alternatives embodying those traits.%
}
For posterity, we will write the set of primitives of $\trait \in \traits$ as
\begin{equation}
\label{eq:prim}
\req(\trait)
	= \setdef{\traitalt \in \traits}{\trait \requires \traitalt},
\end{equation}
For example, if $\traits = \{\brand, \kolor, \model, \type\}$, the implied logical structure readily gives
$\req(\model) = \{\brand, \type\}$ (an Impala is \emph{de facto} a Chevrolet sedan),
and
$\req(\type) = \req(\brand) = \req(\kolor) = \varnothing$ (since none of these traits implicitly requires another).
A family of traits $\traits$ equipped with a logical structure $\requires$ will be called here a \emph{similarity structure} on $\elems$, and we will denote it as $\lattice \equiv \lattice(\traits,\requires)$.

\paragraph{Combining traits.}

We now proceed to detail the procedure of combining two or more traits, leading to clusters like ``Ford compacts'', ``Toyota sedans'', etc.
In this case, a priority ranking gives rise to a nested succession of ``composite'' traits, which can be seen as a \emph{decision tree:}
at each step an additional trait is added to the mix, resulting in finer and finer partitions of $\elems$ until all traits are present.
Formally, define the \emph{combination} (or \emph{logical conjunction}) of two traits $\trait,\traitalt \in \traits$ as
\begin{equation}
\label{eq:conjunction}
\trait \wedge \traitalt
	= \setdef{\class\cap\alt\class}{\class \in \trait, \alt\class \in \traitalt},
\end{equation}
\ie as the coarsest partition of $\elems$ that refines both $\trait$ and $\traitalt$.
Intuitively, this means that $\trait \wedge \traitalt$ groups together all alternatives that combine an instance of $\trait$ with an instance of $\traitalt$.%
\footnote{%
For book-keeping reasons, traits that are obtained this way will be called \emph{composite traits}, to differentiate them from their ``primitive'' counterparts $\trait \in \traits$:
for example, ``Toyota sedan'' is the instance of the composite trait ``\type and \brand'' which describes all cars that are of \brand ``Toyota'' and \type ``sedan'' (the two primitives).%
}
Then, any subset $\subtraits$ of $\traits$ gives rise to a composite trait via the operation 
\begin{equation}
\label{eq:subs2traits}
\txs
\subtraits
	\mapsto \conj(\subtraits) \equiv \bigwedge_{\trait \in \subtraits}\trait.
\end{equation}
Under $\conj$, the empty set $\varnothing$ is mapped to the combination of \emph{no} traits, \ie
the partition of $\elems$ that refines no other partition, namely
\(
\conj(\varnothing)
	= \bigwedge_{\trait \in \varnothing} \trait
	\equiv \one.
\)
Then, including more traits in $\subtraits$ makes the resulting partition $\conj(\subtraits)$ finer and finer, until there are no more traits left to tally.
Hence, the finest partition induced by a family of traits $\traits = \{\trait_{\iTrait}\}_{\iTrait=1}^{\nTraits}$ is
\begin{equation}
\txs
\conj(\traits)
	= \setdef
		{\class_{1} \cap \dotsm \cap \class_{\nTraits}}
		{\class_{1} \in \trait_{1},\dotsc, \class_{\nTraits} \in \traits_{\nTraits}},
\end{equation}
\ie the combination of all primitive traits $\trait \in \traits$.%
\footnote{%
A priori, the partition $\conj(\traits)$ need not coincide with the finest possible partition of $\elems$.
For example, if the set of alternatives contains a yellow 2012 Chevy Impala and a yellow 2016 Impala, these two Impalas cannot be distinguished by either \type, \brand, \kolor, or \model;
one would need to consider a new trait, \yeer, to tell them apart.
That said, if two alternatives are indistinguishable relative to \emph{all} possible traits considered by an agent, we posit that they are, in fact, equivalent for said agent.
By this token, we will be assuming in what follows that the set of primitive traits $\traits = \{\trait_{\iTrait}\}_{\iTrait=1}^{\nTraits}$ provides a \emph{complete description} of $\elems$, \ie
\(
\conj(\traits)
	= \bigwedge_{\trait\in\traits} \trait
	= \zero.
\)
}

All this leads to a more explicit description:
Given a priority ranking ``$\higher$'' on a set of traits $\traits$, we can define a sequence of \emph{nested} partitions:
At the $\iLvl$-th level of the hierarchy we have the combination of $\iLvl$ traits, while a path connecting the root node of $\lattice$ to said node represents the order in which these traits were introduced.
For example, the instances of the composite trait ``\type and \brand'' (\eg ``Ford wagons'' or ``Toyota sedans'') could be reached by first prescribing the car's \type and then its \brand, or the other way around.
Either order reflects a certain priority so, more broadly, a priority ranking corresponds to the order in which traits are progressively combined to form composite ones.
In more precise language, since $\conj(\subtraits) \refines \conj(\alt\subtraits)$ when $\alt\subtraits \subseteq \subtraits$ (see \cref{thm:conjunction} for a formal statement and proof), we conclude that
\begin{equation}
\label{eq:nodes2traits}
\one
	= \conj(\subtraits_{0})
	\refined \dotsm
	\refined \conj(\subtraits_{\iLvl})
	\refined \dotsm
	\refined \conj(\subtraits_{\nTraits})
	= \zero.
\end{equation}
\ie
\begingroup
\itshape
a family of traits $\traits$ on $\elems$ equipped with a priority ranking respecting its logical structure
gives rise to a nested sequence of partitions of $\elems$ which also respects the logical structure of $\traits$.
\endgroup
This reduction of a family of traits to a sequence of nested partitions\textemdash a \emph{tree structure} on $\elems$\textemdash is what forms the core of our reduction to the nested model.

\paragraph{Further properties of combinations of traits}

Recall here that the combination $\bigwedge_{\trait \in \subtraits} \trait$ of a subset of traits $\subtraits \subseteq \traits$ is defined as the coarsest partition of $\elems$ that refines all traits $\trait \in \subtraits$.
Under this definition, combining traits is easily seen to be commutative and associative;
moreover, if we set
\begin{alignat}{3}
\label{eq:identities}
&\one
	= \{\elems\}
	\quad
	&\text{and}
	\quad
&\zero
	= \setdef{\{\elem\}}{\elem \in \elems},
\intertext{we also have}
&\one \wedge \trait
	= \trait
	\quad
	&\text{and}
	\quad
&\zero \wedge \trait
	= \zero.
\end{alignat}
for all $\trait \in \elems$.
Thus, from an algebraic viewpoint, combining traits behaves much like multiplication:
the coarsest partition of $\elems$ acts as an identity element for $\wedge$, while its finest partition acts as an absorbing element.
Finally, it is also easy to check that ``$\conj$'' preserves unions in the sense that
\begin{equation}
\label{eq:union2conjunction}
\conj(\subtraits_{1}\cup\subtraits_{2})
	= \conj(\subtraits_{1}) \wedge \conj(\subtraits_{2})
\end{equation}
for all subsets $\subtraits_{1}, \subtraits_{2}$ of $\traits$.
In other words, combining traits satisfies the following properties for all $\trait, \traitalt, \traitaltalt \in \traits$:
\begin{enumerate}
\item
\emph{Commutativity:}
$\trait \wedge \traitalt = \traitalt \wedge \trait$.
\item
\emph{Associativity:}
$\trait \wedge \traitalt \wedge \traitaltalt = \trait\wedge(\traitalt \wedge \traitaltalt) = (\trait \wedge \traitalt) \wedge \traitaltalt$.
\item
\emph{Idempotency:}
$\trait \wedge \trait = \trait$.
\end{enumerate}

A key feature of the above is the interplay between primitives and conjunctions.
Since a trait refines all its primitives ($\trait \refines \traitalt$ whenever $\trait \requires \traitalt$), we also have $\trait \wedge \traitalt = \trait$ whenever $\trait \requires \traitalt$;
in other words, the conjunction of a trait with any of its primitives returns the original trait unchanged.
This suggests that the set of distinct composite traits that can be obtained from a set of ``primitive'' traits $\traits$ is intimately linked to the prerequisite structure of $\traits$.
To formalize this, fix a logical structure $\requires$ on $\traits = \{\trait_{\iTrait}\}_{\iTrait=1}^{\nTraits}$, and write
\begin{align}
\label{eq:k-span}
\cspan^{\iLvl}(\traits)
	&= \setdef
	{\conj(\{\trait_{\iTrait_{1}}, \dotsm , \trait_{\iTrait_{\iLvl}}\})}
		{\iTrait_{1},\dotsc, \iTrait_{\iLvl} = 1,\dotsc, \nTraits}\\
	&= \setdef
		{\trait_{\iTrait_{1}} \wedge \dotsm \wedge \trait_{\iTrait_{\iLvl}}}
		{\iTrait_{1},\dotsc, \iTrait_{\iLvl} = 1,\dotsc, \nTraits}\notag
\end{align}
for the set of composite traits obtained by taking $\iLvl$-tuple conjunctions of elements of $\traits$ ($\iLvl = 0,\dotsc, \nTraits$).
Subsequently, putting all these traits together, write
\begin{equation}
\label{eq:span}
\cspan
	\equiv \cspan(\traits)
	= \union_{\iLvl=0}^{\nTraits} \cspan^{\iLvl}(\traits)
\end{equation}
for the collection of \emph{all} composite traits generated by the primitive traits of $\traits$.
Of course, as we noted above, it is possible that different combinations of traits of $\traits$ might lead to the \emph{same} composite trait when primitives are involved (since adding primitives of already-included traits does not refine a composite trait).

In principle, the logical structure of $\traits$ might preclude some element of $\cspan$\textemdash i.e., a conjunction of an \emph{arbitrary} group of traits\textemdash from being obtained as the conjunction $\conj(\subtraits)$ of the traits in a node $\subtraits\in \lattice$ of the decision lattice.
In fact, our next result shows that this possibility is never realized:
 the set of composite traits $\cspan \equiv \cspan(\traits)$ obtained from primitive traits is naturally embedded in the associated similarity structure $\lattice \equiv \lattice(\traits,\requires)$.

\begin{theorem}
\label{thm:conjunction}
Let $\traits = \{\trait_{\iTrait}\}_{\iTrait=1}^{\nTraits}$ be a family of traits equipped with a logical structure ``$\requires$''.
Then, with notation as above, we have:
\begin{enumerate}
\addtolength{\itemsep}{\smallskipamount}
\item
\label[part]{part:lattice}
Every element of $\cspan$ is of the form $\conj(\subtraits)$ for some $\subtraits \in \lattice$;
concretely,
\begin{equation}
\label{eq:lattice2span}
\cspan
	= \setdef{\conj(\subtraits)}{\subtraits \in \lattice}
\end{equation}

\item
\label[part]{part:hom-order}
For all $\subtraits, \alt\subtraits \in \lattice$, we have
\begin{equation}
\label{eq:hom-order}
\conj(\subtraits) \refines \conj(\alt\subtraits)
	\quad
	\text{whenever}
	\quad
	\subtraits \supseteq \alt\subtraits.
\end{equation}
\end{enumerate}
In words, the conjunction operator $\conj$ is a surjective, order-preserving map \textendash\ a lattice epimorphism \textendash\ from $(\lattice,\supseteq)$ to $(\cspan,\refines)$.
\end{theorem}

\begin{proof}
To begin, we claim that $\req(\trait) \wedge \trait = \trait$ for all $\trait \in \traits$:
indeed, since $\trait \refines \traitalt$ for all $\traitalt \in \req(\traits)$, this is a simple consequence of 
the associativity of $\wedge$ and 
the fact that $\trait \wedge \traitalt = \trait$ whenever $\trait \refines \traitalt$.

With this auxiliary result at hand, let $\trait_{\iTrait_{1}} \wedge \dotsm \wedge \trait_{\iTrait_{\iLvl}} \in \cspan$ be a composite trait generated by the primitive traits of $\traits$.
Then, letting $\subtraits_{\jTrait} = \req(\trait_{\iTrait_{\jTrait}})\cup\{\trait_{\iTrait_{\jTrait}}\}$ for all $\jTrait=1,\dotsc,\iLvl$, we get
\begin{flalign}
\trait_{\iTrait_{1}} \wedge \dotsm \wedge \trait_{\iTrait_{\iLvl}}
	&= (\req(\trait_{\iTrait_{1}}) \wedge \trait_{\iTrait_{1}})
	\wedge
	\dotsm
	\wedge
	(\req(\trait_{\iTrait_{\iLvl}}) \wedge \trait_{\iTrait_{\iLvl}})
	\notag\\
	&= \conj(\req(\trait_{\iTrait_{1}})\cup\{\trait_{\iTrait_{1}}\})
		\wedge \dotsm \wedge
		\conj(\req(\trait_{\iTrait_{\iLvl}})\cup\{\trait_{\iTrait_{\iLvl}}\})
	\notag\\
	&= \conj(\subtraits_{1}) \wedge \dotsm \wedge \conj(\subtraits_{\iLvl})
	\notag\\
	&= \conj(\subtraits_{1}\cup\dotsm\cup\subtraits_{\iLvl}),
\end{flalign}
where the second and fourth lines follow directly from \eqref{eq:union2conjunction}.
Note also that we have $\req(\trait_{\iTrait_{\jTrait}}) \in \lattice$ for all $\jTrait = 1,\dotsc,\iLvl$;
moreover, since $\trait_{\iTrait_{\jTrait}} \in \new{\req(\trait_{\iTrait_{\jTrait}})}$ by definition, it follows that $\subtraits_{\jTrait} \in \lattice$ for all $\jTrait = 1,\dotsc,\iLvl$.
Thus, with $\lattice$ closed under unions, we get $\subtraits \equiv \subtraits_{1}\cup\dotsm\cup\subtraits_{\iLvl}$, and our claim follows.

Finally, for \cref{part:hom-order}, let $\alt\subtraits = \{\trait_{\iTrait_{1}},\dotsc,\trait_{\iTrait_{\alt\iLvl}}\} \subseteq  \{\trait_{\iTrait_{1}},\dotsc,\trait_{\iTrait_{\iLvl}}\} = \subtraits$ for some $\alt\iLvl \leq \iLvl$.
Then, by \eqref{eq:union2conjunction}, we obtain
\begin{flalign}
\conj(\subtraits)
	&= \trait_{\iTrait_{1}} \wedge \dotsc \wedge \trait_{\iTrait_{\iLvl}}
	= (\trait_{\iTrait_{1}} \wedge \dotsc \wedge \trait_{\iTrait_{\alt\iLvl}})
		\wedge (\trait_{\iTrait_{\alt\iLvl+1}} \wedge \dotsc \wedge \trait_{\iTrait_{\iLvl}})
	\notag\\
	&\refines \trait_{\iTrait_{1}} \wedge \dotsc \wedge \trait_{\iTrait_{\alt\iLvl}}
	= \conj(\alt\subtraits),
\end{flalign}
and our proof is complete.
\end{proof}

This theorem serves as a sanity check that the nested presentation of a similarity structure is, indeed, faithful.

%% file: Traits.tex
Consider again a finite set of \emph{alternatives} (or \emph{actions}) indexed by $\elem \in \elems = \{1,\dotsc, \nElems\}$.
Based on their innate traits, these alternatives can be clustered in different ways so that each subgroup shares a common attribute:
in our running example from the introduction, a collection of cars 
could be categorized
by brand into Fords, Chevrolets, Toyotas, etc.;
by type into sedans, compacts, SUVs, and so forth;
by powertrain into gas, hybrid, or electric;
etc. 
To formalize this idea, a \emph{trait} is defined to be a partition $\trait$ of $\elems$ into pairwise disjoint subsets\textemdash called \emph{classes}\textemdash each corresponding to a different \emph{instance} of said trait.
Thus, in the catalogue of cars considered above, \type and \brand would represent different traits, while Ford and Toyota would be different instances of the trait \brand.

Accordingly, to define a similarity structure on $\elems$, we will assume %
that $\elems$ is equipped with a family of traits $\traits = \{\trait_{\iTrait}\}_{\iTrait=1}^{\nTraits}$.
Of course, two or more traits can be combined, leading to clusters like ``Ford compacts'', ``Toyota sedans'', etc.
Formally, the \emph{combination} (or \emph{logical conjunction})
of two traits $\trait,\traitalt \in \traits$ is defined to be the ``composite trait''
$\trait \wedge \traitalt
	= \setdef{\class\cap\alt\class}{\class \in \trait, \alt\class \in \traitalt}\,; $
in words, $\trait \wedge \traitalt$ simply groups together all alternatives that combine an instance of $\trait$ with an instance of $\traitalt$. %
More generally, any subset $\subtraits$ of $\traits$ gives rise to a composite trait via the operation 
$\txs
\subtraits
	\mapsto \conj(\subtraits) \equiv \bigwedge_{\trait \in \subtraits}\trait \,.$
Then, to quantify all this, we will associate to each subset $\subtraits$ of $\traits$ a non-negative \emph{heterogeneity score} $\scalegeneral_{\subtraits}$:
the smaller this score, the more similar any two alternatives agreeing on the traits in $\subtraits$ are.
We naturally posit that $\scalegeneral_{\traits} = 0$ and that $\scalegeneral_{\subtraits} \leq \scalegeneral_{\subtraits'} $ when $\subtraits' \subseteq \subtraits$.

Moving forward, when selecting an alternative, one may first focus on the type of a car before worrying about its brand, or one could do the reverse.
To model this, we define a \emph{priority ranking} on a family of traits $\traits$ to be a strict total order $\higher$ (read: ``\emph{has higher priority than}'').
Given a priority ranking, we can define a sequence $(\subtraits_{\iLvl})_{\iLvl = 0}^{\nTraits}$, where $\subtraits_{0} = \varnothing$, $\subtraits_{\nTraits} = \traits$, and $\subtraits_{\iLvl}$ is the set composed of the first $\iLvl$ traits.
We thus obtain a nested sequence of partitions $\conj(\subtraits_{0}), \dots, \conj(\subtraits_{\nTraits})$, each successively refining its predecessor.
This sequence is nothing but a \emph{tree structure} on $\elems$, and, taken together, the family of traits, heterogeneity scores and the priority rank comprise what we call a \emph{similarity structure} on $\elems$.

In reference to the tree compatible loss model described in \Cref{sec:preliminaries},
if $\subtraits_j$ is the combination of traits associated with the $j$-th level of the tree, then $\sigma_{j+1}$ is simply $\scalegeneral_{\subtraits_j}$.\footnote{The $+1$ increment is due to the way $\similarity(\cdot,\cdot)$ is defined.}
Hence, the similarity scales depend on the priority ranking, and can hence affect the regret bounds.
In the next section, we provide more technical details involving the aforementioned process of constructing the tree, and show how the number of logically admissible trees can be reduced.